\documentclass[twoside,11pt]{article}
\usepackage{jmlr2e}
\pdfoutput=1

\usepackage{amsmath,amssymb,mathtools,subfigure,graphicx,url,multirow,multicol,array,subfigure}

\usepackage[usenames,dvipsnames]{color}

\usepackage[ruled]{algorithm2e}

\usepackage{epstopdf}

\usepackage{soul}

\usepackage{pdflscape,afterpage,capt-of,lipsum}

\newcommand{\blue}[1]{\textcolor{blue}{#1}}
\newcommand{\red}[1]{\textcolor{red}{#1}}

\newcommand{\mcl}[2]{\multicolumn{#1}{c|}{#2}}

\DeclareMathOperator*{\argmin}{arg\,min}
\newcommand{\defeq}{\mathrel{\mathop:}=}

\def\R{\mathbb{R}}
\def\B{\mathcal{S}^{d-1}}

\def\dd{\mathrm{d}}

\def\x{\mathbf{x}}
\def\y{\mathbf{y}}
\def\v{\mathbf{v}}
\def\w{\mathbf{w}}
\def\u{\mathbf{u}}

\def\X{\mathcal{X}}

\def\O{\mathcal{O}}

\def\vs{\mathbf{v}^\star}
\def\bs{b^\star}

\def\Th{\Theta}

\def\conv{\mathrm{conv}\,}

\def\MDP{MDP$^2$\ }
\def\S3VM{S$^3$VM}

\def\phiSSL{\phi_{\mathrm{SSC}}}
\def\phiUL{\phi_{\mathrm{CL}}}
\def\fUL{f_{\mathrm{CL}}}
\def\fSL{f_{\mathrm{SSC}}}

\def\hY{B}

\def\margin{\mathrm{margin}}

\title{Minimum Density Hyperplanes} \date{}

\jmlrheading{17}{2016}{1-33}{6/15; Revised 9/16}{9/16}{Nicos G. Pavlidis, David P. Hofmeyr, and Sotiris K. Tasoulis}

\ShortHeadings{Minimum Density Hyperplanes}{Pavlidis, Hofmeyr and Tasoulis}
\firstpageno{1}

\begin{document}


\author{\name Nicos G.\ Pavlidis \email n.pavlidis@lancaster.ac.uk \\
\addr Department of Management Science \\
Lancaster University\\
Lancaster, LA1 4YX, UK
\AND
\name David P.\ Hofmeyr \email d.hofmeyr@lancaster.ac.uk \\
\addr Department of Mathematics and Statistics\\
Lancaster University\\
Lancaster, LA1 4YF, UK
\AND
\name Sotiris K.\ Tasoulis \email s.tasoulis@ljmu.ac.uk \\
\addr Department of Applied Mathematics\\
Liverpool John Moores University, \\
Liverpool, L3 3AF, UK
}

\editor{Andreas Krause}

\maketitle

\begin{abstract}%

Associating distinct groups of objects (clusters) with contiguous regions of
high probability density (high-density clusters), is central to many
statistical and machine learning approaches to the classification of unlabelled
data. We propose a novel hyperplane classifier for clustering and
semi-supervised classification which is motivated by this objective.  The
proposed \textit{minimum density hyperplane} minimises the integral of the
empirical probability density function along it, thereby avoiding intersection
with high density clusters.  We show that the minimum density and the maximum
margin hyperplanes are asymptotically equivalent, thus linking this approach to
maximum margin clustering and semi-supervised support vector classifiers.  We
propose a projection pursuit formulation of the associated optimisation problem
which allows us to find minimum density hyperplanes efficiently in practice,
and evaluate its performance on a range of benchmark data sets.  The proposed
approach is found to be very competitive with state of the art methods for
clustering and semi-supervised classification.

\end{abstract}


\section{Introduction}

We study the fundamental learning problem: {\em Given a random sample from an
unknown probability distribution with no, or partial label information,
identify a separating hyperplane that avoids splitting any of the distinct
groups (clusters) present in the sample.}
We adopt the cluster definition given by~\citet[chap.~11]{Hartigan1975}, in which a {\em
high-density cluster} is defined as a maximally connected component of the
level set of the probability density function, $p(\x)$, at level~$c \geqslant
0$,
\[\mathrm{lev}_c p(\x) = \left\{ \x \in \R^d \, \big{\vert}\, p(\x) > c\right\}.\]
An important advantage of this approach over other methods is that it is
well founded from a statistical perspective, in the sense that a well-defined population quantity 
is being estimated.

However, since $p(\x)$ is typically unknown, detecting high-density clusters
necessarily involves estimates of this function, and standard approaches to
nonparametric density estimation are reliable only in low dimensions. 
A number of existing \textit{density clustering} algorithms approximate the
level sets of the empirical density through a union of spheres around points
whose estimated density exceeds a user-defined
threshold~\citep{Walther1997,CuevasFF2000,CuevasFF2001,RinaldoW2010}.
The choice of this threshold affects both the shape and number of detected
clusters, while an appropriate threshold is typically not known in advance.
The performance of these methods deteriorates sharply as dimensionality
increases, unless the clusters are assumed to be clearly
discernible~\citep{RinaldoW2010}.
An alternative is to consider the more specific problem of allocating
observations to clusters,
which shifts the focus to local properties of the density, rather than its
global approximation. The central idea underlying such methods is that if a
pair of observations belong to the same cluster they must be connected through
a path traversing only high-density regions.
Graph theory is a natural choice to address this type of problem.
\cite{AzzaliniT2007,StuetzleN2010} and \cite{MenardiA2014} have recently
proposed algorithms based on this approach. Even these approaches however are
limited to problems of low dimensionality by the standards of current
applications~\citep{MenardiA2014}.

An equivalent formulation of the density clustering problem is to assume that
clusters are separated through contiguous regions of low probability density;
known as the {\em low-density separation} assumption. 
In both clustering and semi-supervised classification,
identifying the hyperplane with the maximum margin is considered a direct implementation
of the low-density separation approach.
Motivated by the success of support vector machines (SVMs) in classification, maximum
margin clustering (MMC)~\citep{XuNLS2004}, seeks the
maximum margin hyperplane to perform a binary partition (bi-partition) of unlabelled data.
MMC can be equivalently viewed as seeking the binary labelling of the data
sample that will maximise the margin of an SVM
estimated using the assigned labels.

In a plethora of applications data can be collected cheaply and automatically,
while labelling observations is a manual task that can be performed for a small
proportion of the data only.
Semi-supervised classifiers attempt to exploit the abundant unlabelled data to
improve the generalisation error over using only the scarce labelled examples.
Unlabelled data provide additional information about the marginal density,
$p(\x)$, but this is beneficial only insofar as it improves the inference of
the class conditional density,~$p(\x|y)$.
Semi-supervised classification relies on the assumption that a relationship
between $p(\x)$ and $p(\x|y)$ exists.
The most frequently assumed relationship is that high-density clusters are
associated with a single class (cluster assumption), or equivalently that class
boundaries pass through low-density regions (low-density separation
assumption).
The most widely used semi-supervised classifier based on the low-density
separation assumption is the semi-supervised support vector machine
(S$^3$VM)~\citep{VapnikS1977,Joachims1999,ChapelleZ2005}.
S$^3$VMs implement the low-density separation assumption by partitioning the
data according to the maximum margin hyperplane with respect to both labelled
and unlabelled data.

Encouraging theoretical results for semi-supervised classification have been
obtained under the cluster assumption. 
If $p(\x)$ is a mixture of class conditional distributions,
\citet{CastelliC1995,CastelliC1996} have shown that the generalisation error
will be reduced exponentially in the number of labelled examples if the mixture
is identifiable.
More recently, \cite{SinghNZ2008} showed that the mixture components can be identified if
$p(\x)$ is a mixture of a finite number of smooth density functions, and the
separation between mixture components is large.
\cite{Rigollet2007} considers the cluster assumption in a nonparametric
setting, that is in terms of density level sets, and shows that the
generalisation error of a semi-supervised classifier decreases exponentially
given a sufficiently large number of unlabelled data. 
However, the cluster assumption is difficult
to verify with a limited number of labelled examples.
Furthermore, the algorithms proposed by~\citet{Rigollet2007}
and~\citet{SinghNZ2008} are difficult to implement efficiently even if the
cluster assumption holds. This renders them impractical for real-world
problems~\citep{JiYLJH2012}.

Although intuitive, the claim that maximising the margin over (labelled and)
unlabelled data is equivalent to identifying the hyperplane that goes through
regions with the lowest possible probability density has received surprisingly
little attention.  The work of~\cite{BenDavidLPS2009} is the only attempt we
are aware of to theoretically investigate this claim.
\citet{BenDavidLPS2009} quantify the notion of a low-density separator by
defining the {\em density on a hyperplane}\/, as the integral of the
probability density function along the hyperplane.
They study the existence of universally consistent algorithms to compute the
hyperplane with minimum density.
The maximum hard margin classifier is shown to be consistent only in one
dimensional problems.
In higher dimensions only a soft-margin algorithm is a consistent estimator of
the minimum density hyperplane.  \citet{BenDavidLPS2009} do not provide an
algorithm to compute low density hyperplanes.

This paper introduces a novel approach to clustering and semi-supervised
classification which directly identifies low-density hyperplanes in the finite
sample setting.
In this approach the density on a hyperplane criterion proposed
by~\cite{BenDavidLPS2009} is directly minimised with respect to a kernel
density estimator that employs isotropic Gaussian kernels.
The density on a hyperplane provides a uniform upper bound on the value of the
empirical density at points that belong to the hyperplane. This bound is tight
and proportional to the density on the hyperplane. Therefore, the smallest
upper bound on the value of the empirical density on a hyperplane is achieved
by hyperplanes that minimise the density on a hyperplane criterion.
An important feature of the proposed approach is that the density on a
hyperplane can be evaluated exactly through a one-dimensional kernel density
estimator, constructed from the projections of the data sample onto the vector
normal to the hyperplane.
This renders the computation of minimum density hyperplanes tractable even in
high dimensional applications.

We establish a connection between the minimum density hyperplane and the
maximum margin hyperplane in the finite sample setting.
In particular, as the bandwidth of the kernel density estimator is reduced
towards zero, the minimum density hyperplane converges to the maximum margin
hyperplane. An intermediate result establishes that there exists a positive
bandwidth such that the partition of the data sample induced by the minimum
density hyperplane is identical to that of the maximum margin hyperplane.
%

The remaining paper is organised as follows:
The formulation of the minimum density hyperplane problem as well as basic
properties are presented in Section~\ref{sec:formulation}.
Section~\ref{sec:MaxMarg}  establishes the connection between minimum density
hyperplanes and maximum margin hyperplanes. 
Section~\ref{sec:Methodology} discusses the estimation of
minimum density hyperplanes and the computational complexity of the resulting
algorithm.
Experimental results are presented in
Section~\ref{sec:Exper}, followed by concluding remarks and future research
directions in Section~\ref{sec:concl}.

\section{Problem Formulation}\label{sec:formulation}

We study the problem of estimating a hyperplane to partition a finite data set,
$\X = \{ \x_i\}_{i=1}^n \subset \R^d$, without splitting any of the
high-density clusters present.
We assume that $\X$ is an i.i.d. sample
of a random variable $\mathbf{X}$ on $\R^d$, with unknown probability density function $p:\R^d \to \R^+$.
A hyperplane is defined as $H(\v,b) \defeq \{ \x \in \R^d \,|\, \v \cdot \x = b\}$, where without loss of
generality we restrict attention to hyperplanes with unit normal vector, i.e., those parameterised by
$(\v,b) \in \B \times \R$, where $\B = \{ \v \in \R^d \, \big|\, \| \v\|=1\}$. 
Following~\citet{BenDavidLPS2009} we define the {\em density on the hyperplane} $H(\v,b)$
as the integral of the probability density function along the hyperplane,
\begin{equation}\label{eq1}
I(\v,b) \defeq \int_{ H(\v,b)} p(\x) \dd \x.
\end{equation}

We approximate $p(\x)$ through a kernel density estimator with isotropic Gaussian
kernels,
\begin{equation}\label{eq3}
\hat{p}(\x \vert \X, h^2 I) = \frac{1}{n (2 \pi h^2)^{d/2} } \sum_{i=1}^n \exp \left\{ -\frac{ \|\x - \x_i \|^2}{2 h^2} \right\}.
\end{equation}
This class of kernel density estimators has the useful property that the
integral in Equation~(\ref{eq1}) can be evaluated exactly
by projecting $\X$ onto $\v$; constructing a one-dimensional density estimator
with Gaussian kernels and bandwidth~$h$; and evaluating the density at~$b$,
\begin{align}\label{eq4}
\hat I(\v,b| \X, h^2 I) & \defeq \int_{H(\v, b)} \hat{p} \left(\x | \X, h^2 I \right) \dd \x, \nonumber \\
& = \frac{1}{n\sqrt{2\pi h^2 }} \sum_{i=1}^n  \exp \left\{- \frac{ (b - \v \cdot \x_i)^2 }{2 h^2} \right\}
=  \hat{p} \left(b \, | \, \{ \v\cdot \x_i \}_{i=1}^n, h^2 \right).
\end{align}
The univariate kernel estimator $\hat{p} \left(\cdot \, | \, \{ \v\cdot \x_i \}_{i=1}^n, h^2 \right)$
approximates the {\em projected density on~$\v$}, that is,
the density function of the random variable, $X_\v = \mathbf{X} \cdot \v$.
Henceforth we use $\hat{I}(\v,b)$ to approximate $I(\v,b)$.
To simplify terminology we refer to $\hat I(\v,b)$ as the {\em density on
$H(\v,b)$}, or the {\em density integral on $H(\v,b)$}, rather than the empirical
density, or the empirical density integral, respectively.
For notational convenience we write 
$\hat{I}(\v, b)$ for $\hat{I}(\v,b| \X, h^2 I)$, where $\X$ and $h$
are apparent from context.

The following Lemma, adapted from~\cite[Lemma 3]{TasoulisTP2010}, shows that
$\hat{I}(\v,b)$ provides an upper bound for the maximum value of the
empirical density at any point that belongs to the hyperplane.
\begin{lemma}\label{lem:levsetbound}
Let $\X = \{\x_i\}_{i=1}^n \subset \R^d$, and $\hat{p}(\x|\X,h^2 I)$
be a kernel density estimator with isotropic Gaussian kernels.
Then, for any $(\v,b) \in \B\times \R$,
\begin{equation}\label{lm:Tasoulis}
\max_{\x \in H(\v,b)} \hat{p} (\x| \X,h^2 I) \leqslant \left(2 \pi h^2 \right)^{\frac{1-d}{2}} \hat{I}(\v, b),
\quad \textrm{for all } \x \in H(\v, b).
\end{equation}
\end{lemma}
This lemma shows that a hyperplane, $H(\v,b)$, cannot intersect level sets of
the empirical density with level higher than $\left(2 \pi h^2
\right)^{\frac{1-d}{2}} \hat{I}(\v, b)$.
The proof of the lemma relies on the fact that projection contracts distances,
and follows from simple algebra.
In Equation~(\ref{lm:Tasoulis}) equality holds if and only if
there exists $\x \in H(\v,b)$ and $\mathbf{c} \in \R^n$ such that all $\x_i \in
\X$, can be written as $\x_i = \x + c_i \v$.
It is therefore not possible to obtain a uniform upper bound on the value of
the empirical density at points that belong to $H(\v,b)$ that is lower than
$\left(2 \pi h^2 \right)^{\frac{1-d}{2}} \hat{I}(\v, b)$ using only
one-dimensional projections. 
Since the upper bound of Lemma~\ref{lem:levsetbound} is tight and proportional
to $\hat{I}(\v, b)$, minimising the density on the hyperplane leads to the
lowest upper bound on the maximum value of the empirical density along the
hyperplane separator.

To obtain hyperplane separators that are meaningful for clustering and semi-supervised classification,
it is necessary to constrain the set of feasible solutions, because the density
on a hyperplane can be made arbitrarily low by considering a hyperplane that intersects
only the tail of the density. In other words, for any $\v$, $\hat{I}(\v,b)$ can be made arbitrarily low for sufficiently
large~$|b|$.
In both problems the constraints restrict the feasible set
to a subset of the hyperplanes that intersect the
interior of the convex hull of $\X$. 
In detail, let $\conv\X$ denote the convex hull of~$\X$, and assume
$\mbox{Int}(\conv\X)\neq \emptyset$. Define~$C$ to be the set of hyperplanes that
intersect $\mbox{Int}(\conv\X)$,
\begin{equation}\label{eq:domain}
C =\left\{ H(\v, b) \, \Big{\vert} \, (\v, b) \in \B \times \R, \, \exists \mathbf{z} \in \mbox{Int}(\conv
\X) \mbox{ s.t. } \v \cdot \mathbf{z} = b\right\}.
\end{equation}
Then denote by $F$ the set of feasible hyperplanes, where $F \subset C$.
We define the {\em minimum density hyperplane} (MDH), $H(\vs,\bs) \in F$ to satisfy,
\begin{equation}\label{eq2}
\hat{I}(\vs,\bs) = \min_{(\v, b) \vert H(\v, b) \in F} \hat{I}(\v,b).
\end{equation}
In the following subsections we discuss the specific formulations for clustering
and semi-supervised classification in turn.

\subsection{Clustering} \label{ssec:UL}

Since high-density clusters are formed around the modes of~$p(\x)$,
the convex hull of these modes would be
a natural choice to define the set of feasible hyperplanes.
Unfortunately, this convex hull is unknown and difficult to estimate.
We instead propose to constrain the distance of hyperplanes to the origin,~$b$.
Such a constraint is inevitable as for any $\v \in \B$, $\hat{I}(\v,b)$ can
become arbitrarily close to zero for sufficiently large~$|b|$.
Obviously, such hyperplanes are inappropriate for the purposes of
bi-partitioning as they assign all the data to the same partition.
Rather than fixing~$b$ to a constant, we constrain it in the interval, 
\begin{equation}\label{constr1}
F(\v) = \left[ \mu_{\v} - \alpha \sigma_\v, \mu_\v + \alpha \sigma_\v \right],
\end{equation} 
where $\mu_\v$ and $\sigma_\v$ denote the mean and standard deviation, respectively, of the
projections $\{\v \cdot \x_i\}_{i=1}^n$.
The parameter~$\alpha \geqslant 0$, controls the width of the interval, and has a
probabilistic interpretation from Chebyshev's inequality.
Smaller values of $\alpha$ favour more balanced partitions of the data
at the risk of excluding low density hyperplanes that separate
clusters more effectively.
On the other hand, increasing~$\alpha$ increases the risk
of separating out only a few outlying observations.
We discuss in detail how to set this parameter in the experimental results section.
If $\mathrm{Int}( \conv \X) \neq \emptyset$, then there exists $\alpha>0$ 
such that the set of feasible hyperplanes for clustering, $F_{\mathrm{CL}}$, satisfies,
\begin{equation}\label{eq:feasible}
F_{\mathrm{CL}} = \left\{ H(\v,b) \, \Big{\vert} \, (\v,b) \in \B \times \R, \, b \in F(\v) \right\} \subset C,
\end{equation}
where $C$ is the set of hyperplanes that intersect $\mathrm{Int}(\conv \X)$,
as defined in Equation~(\ref{eq:domain}).

The minimum density hyperplane for clustering is the solution to the following constrained
optimisation problem,
\begin{subequations}\label{eq:ULconstr1}
\begin{align}
& \min_{(\v, b) \in  \B \times \R} \  \hat I(\v, b),  \\
\textrm{subject to:} \ \ &	b - \mu_{\v} + \alpha \sigma_\v \geqslant 0, \\
& \mu_{\v} + \alpha \sigma_\v - b \geqslant 0.
\end{align}
\end{subequations}
Since the objective function and the constraints are continuously
differentiable, MDHs can be estimated through
constrained optimisation methods like sequential quadratic programming
(SQP).
Unfortunately the problem of local minima due to the nonconvexity of the
objective function seriously hinders the effectiveness of this approach.

To mitigate this we propose a parameterised optimisation formulation,
which gives rise to a projection pursuit approach. 
Projection pursuit methods optimise a measure of ``interestingness'' of a
linear projection of a data sample, known as the projection index.
For our problem the natural choice of projection index for~$\v$
is the minimum value of the projected density within the feasible region,
$\min_{b \in F(\v)} \hat{I}(\v,b)$.
This index gives the minimum density integral of feasible hyperplanes with normal vector~$\v$.
To ensure the differentiability of the projection index we
incorporate a penalty term into the objective function.
We define the penalised density integral as,
\begin{align}
\fUL(\v,b) & = \hat{I}(\v,b) + \frac{L}{\eta^\epsilon} \max \left\{ 0, \mu_\v  -\alpha \sigma_\v - b, b - \mu_\v  -\alpha \sigma_\v \right\}^{1+\epsilon},\label{eq:PenalisedUL} 
\end{align}
where,
$L = \left(e^{1/2}h^2 \sqrt{2 \pi}\right)^{-1} \geqslant \sup_{b \in \R}
\left| \frac{\partial \hat{I}(\v,b)}{\partial{b}} \right|$, $\epsilon \in (0,1)$ is a constant term that ensures that the
penalty function is everywhere continuously differentiable, and $\eta \in
(0,1)$.
Other penalty functions are possible, but we only consider the above due to its
simplicity, and the fact that its parameters offer a direct interpretation: $L$
in terms of the derivative of the projected density on $\v$; and $\eta$ in terms of the
desired accuracy of the minimisers of $\fUL(\v,b)$ relative to the minimisers
of Equation~(\ref{eq:ULconstr1}), as discussed in the following proposition.

\begin{proposition} \label{prop:minim}

For $\v \in \B$, define, the set of minimisers,
\begin{align}
	B(\v) & =  \argmin_{b \in F(\v)} \hat{I}(\v,b), \\
B_C(\v) & = \argmin_{b \in \R} \fUL(\v,b)
\end{align}
For every $\bs \in B(\v)$ there exists $\bs_C \in B_C(\v)$ such that $| \bs  - \bs_C| \leqslant \eta$.
Moreover, there are no minimisers of $\fUL(\v,b)$ outside the interval $[\mu_v - \alpha \sigma_\v -\eta,  \mu_v + \alpha \sigma_\v + \eta]$,
\[
B_C(\v) \cap \R \backslash [\mu_v - \alpha \sigma_\v -\eta, \mu_v + \alpha \sigma_\v +\eta] = \emptyset.
\]

\end{proposition}

\begin{proof}

Any minimiser in the interior of the feasible region, $\bs \in  B(\v) \cap \mbox{Int}(F(\v))$, also minimises the penalised function,
since $\fUL(\v,b)= \hat{I} (\v, b)$ for all $b \in \mbox{Int}(F(\v))$, hence $\bs \in B_C(\v)$.

Next we consider the case when either or both of the boundary points of $F(\v)$,
$b^- = \mu_v - \alpha \sigma_\v$ and $b^+ = \mu_v + \alpha \sigma_\v$, are contained in $B(\v)$. 
It suffices to show that,
$\fUL(\v,b) > \hat{I}(\v, b^-)$ for all $b < b^- - \eta$, and
$\fUL(\v,b) > \hat{I}(\v, b^+)$ for all $b > b^+ + \eta$.
We discuss only the case $b > b^+ + \eta$ as the treatment of $b<b^- - \eta$ is
identical. Assume that $\hat{I}(\v, b) < \hat{I}(\v, b^+)$ (since in the opposite case
the result follows immediately: $\fUL(\v,b) > \hat{I}(\v, b) > \hat{I}(\v, b^+)$).
From the mean value theorem there exists $\xi \in (b^+,b)$ such that,
\begin{align*}
\hat{I}(\v, b^+) & = \hat{I}(\v, b) - (b - b^+) \left.\frac{ \partial{ \hat{I}(\v,b)} }{\partial{b}} \right|_{b=\xi} \\
& \leqslant \hat{I}(\v,b) + (b - b^+) L \\
& < \hat{I}(\v, b) + \frac{L (b - b^+)^{1+\epsilon}}{\eta^\epsilon} = \fUL(\v,b).
\end{align*}
In the above we used the following facts: $\left.\frac{ \partial{ \hat{I}(\v,b)} }{\partial{b}} \right|_{b=\xi} <0$, $L \geqslant \sup_{b \in \R} \left| \frac{ \partial{ \hat{I}(\v,b)} }{\partial{b}} \right|$, and
$\frac{b - b^+}{\eta}>1$.

\end{proof}

We define the 
projection index for the clustering problem as the minimum of the penalised density integral,
\begin{equation}\label{eq:PrIndexUL}
\phiUL(\v) = \min_{b \in \R} \fUL(\v,b).
\end{equation}
Since the optimisation problem of Equation~(\ref{eq:PrIndexUL}) is one-dimensional 
it is simple to compute the set of global minimisers $B_C(\v)$.
As we discuss in Section~\ref{sec:Methodology},
this is necessary to compute directional derivatives of the projection index,
as well as, to determine whether $\phiUL$ is differentiable.
We call the optimisation of $\phiUL$, {\em minimum density projection pursuit}
(MDP$^2$).
For each $\v$, \MDP considers only the optimal choice of $b$. This enables it to
avoid local minima of $\hat{I}(\v, \cdot)$.
Most importantly \MDP is able to accommodate a discontinuous change in the
location of the global minimiser(s), $\argmin_{b \in\R} \fUL(\v,b)$, as $\v$ changes.
Neither of the above can be achieved when the optimisation is jointly over $(\v,b)$ as
in the original constrained optimisation problem, Equation~(\ref{eq:ULconstr1}).
The projection index $\phiUL$ is continuous, but it
is not guaranteed to be everywhere differentiable when $B_C(\v)$ is not a
singleton. The resulting optimisation problem is therefore nonsmooth and
nonconvex.

To illustrate the effectiveness of \MDP to estimate MDHs,
we compare this approach with a direct optimisation of the constrained problem
given in Equation~(\ref{eq:ULconstr1}) using SQP.
To enable visualisation we consider the two-dimensional~S1
data set~\citep{FrantiV2006}, constructed by sampling from a Gaussian
mixture distribution with fifteen components, where each component corresponds to a cluster. 
\begin{figure}[t]
\begin{center}
\subfigure[SQP]{\includegraphics[scale=0.3]{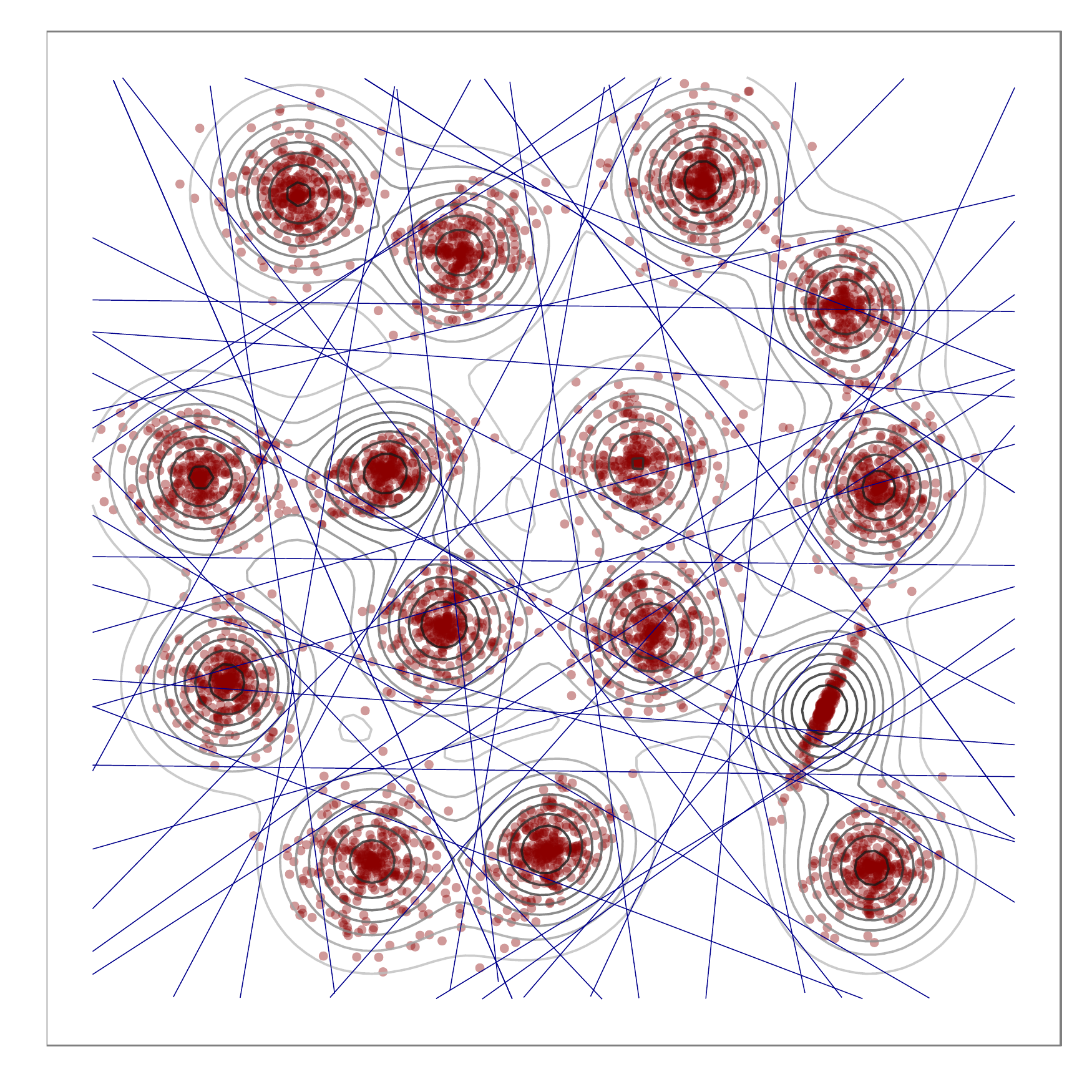}\label{fig:SQP1}} \hspace{0.5cm}
\subfigure[\MDP]{\includegraphics[scale=0.3]{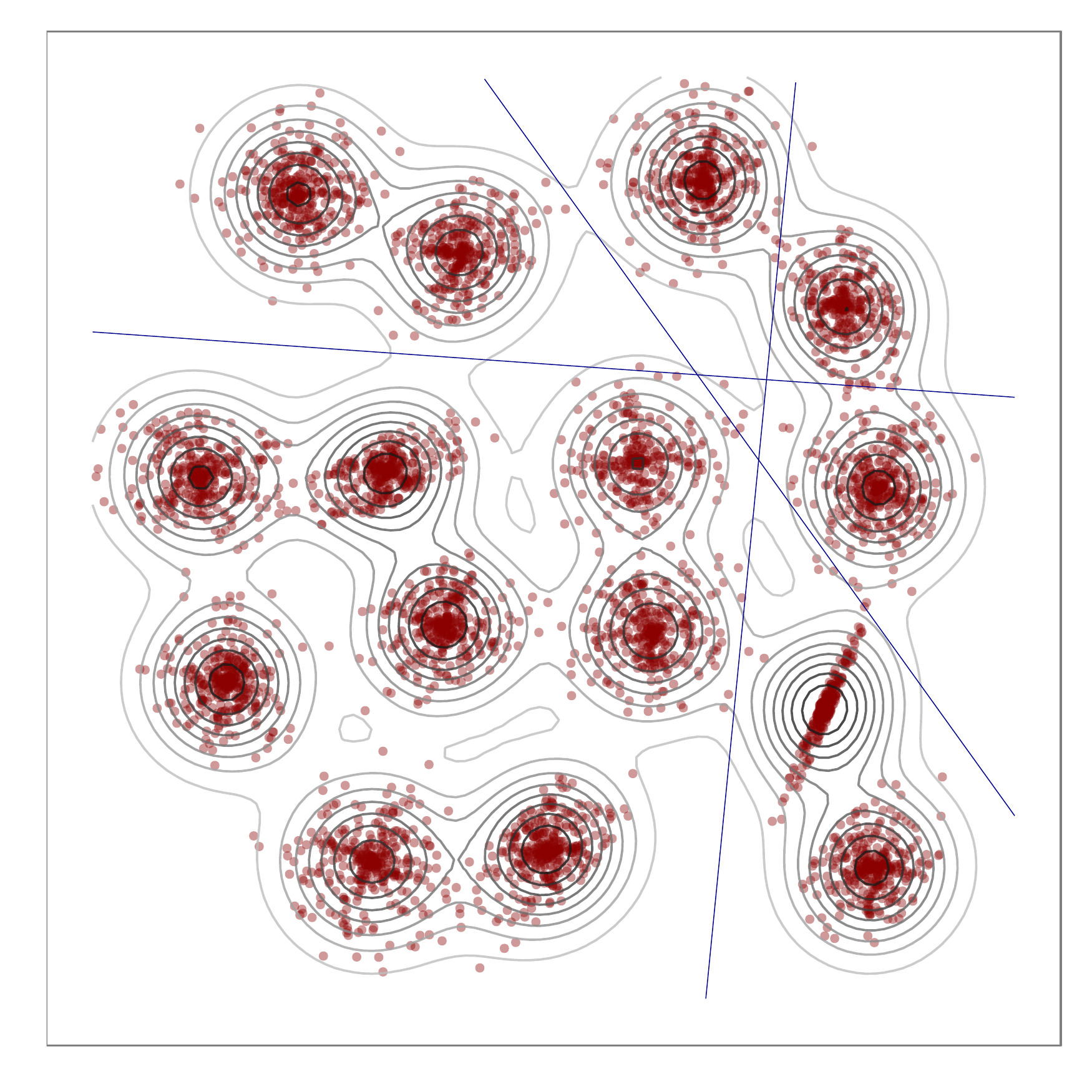}\label{fig:MDP1}}
\end{center}
\caption{Binary partitions induced by 100 MDHs estimated through SQP and
\MDP}\label{fig:SQPvsMDP}
\end{figure} 
Figure~\ref{fig:SQPvsMDP} depicts the
MDHs obtained over 100 random initialisations of SQP and MDP$^2$.
It is evident that SQP frequently yields hyperplanes that intersect regions with
high probability density thus splitting clusters.
As SQP always converged in these experiments the poor performance is solely
due to convergence to local minima.
In contrast, \MDP converges to three different solutions over the 100
experiments, all of which induce high quality partitions,
and none intersects a high-density cluster. 
\begin{figure}\centering
\includegraphics[scale=0.4]{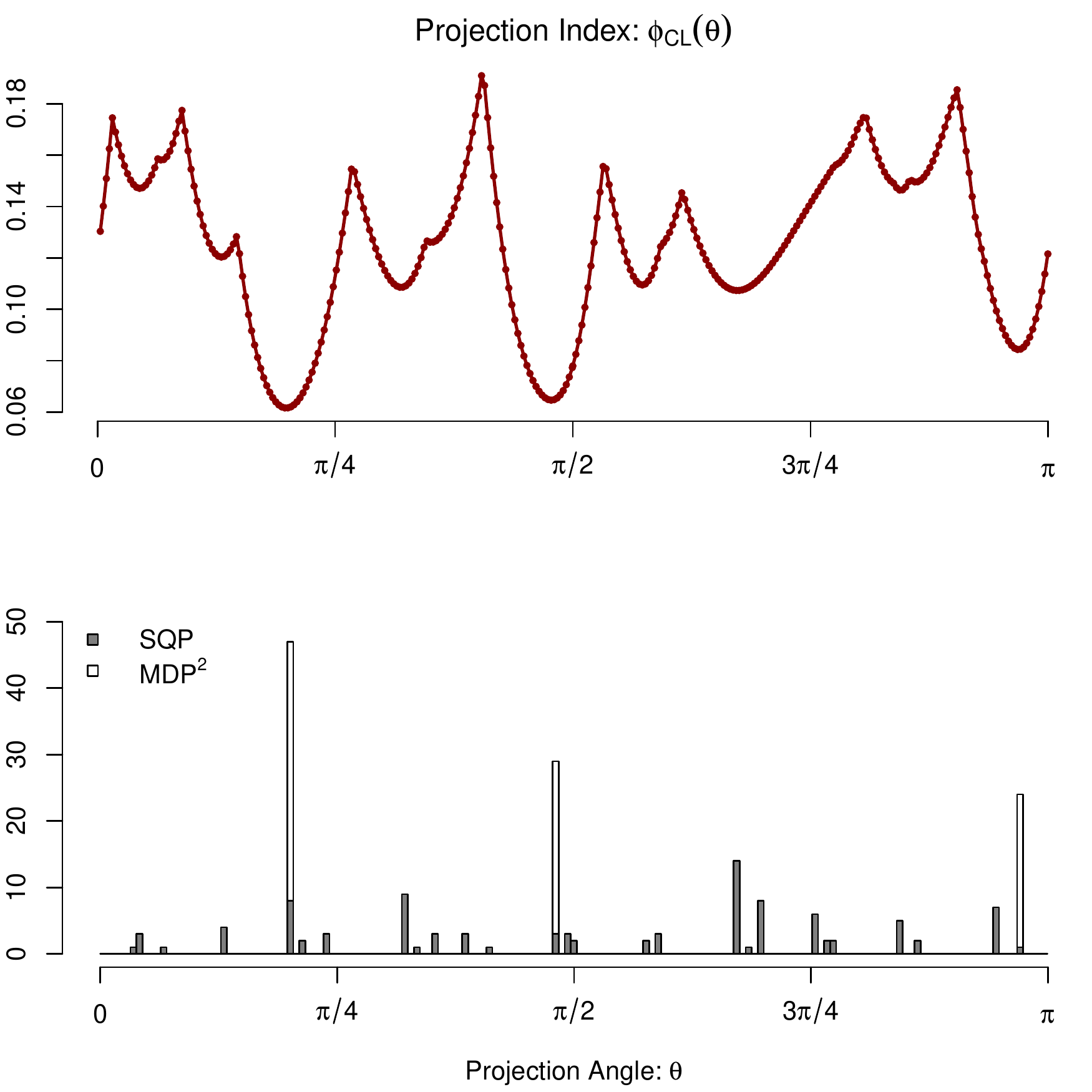}
\caption{Projection index for S1 data set and solutions obtained through SQP and MDP$^2$}\label{fig:hist}
\end{figure}
In polar coordinates any $\v \in \mathcal{S}^1$ can be parameterised
through a single projection angle. Using this parameterisation, the upper plot
of Figure~\ref{fig:hist} depicts the value of the projection index,
$\phiUL(\v(\theta))$, for $\theta \in [0,\pi]$. The lower plot of the figure
provides histograms of the distribution of the solutions (locally optimal
projection angles) obtained over the 100 experiments with SQP (grey) and
MDP$^2$ (white).  The figure shows that $\phiUL(\v)$ is continuous but not
everywhere differentiable. The solution most frequently obtained through
MDP$^2$ corresponds to the global optimum, while the only other two solutions
identified are the local minimisers with the next two lowest function values.
In contrast SQP converges to a much wider range of solutions. Note that this method is
not guaranteed to identify the optimal value of $b$ for any $\v(\theta)$ and
this indeed occurs in this example. Therefore the value of $\phiUL(\v)$
is a lower bound for the function values of the minimisers identified through
SQP.

\subsection{Semi-Supervised Classification}\label{ssec:SSL}

In semi-supervised classification labels are available for a subset of
the data sample. The resulting classifier needs to predict as accurately as
possible the labelled examples, while avoiding intersection with high-density
regions of the empirical density.
The MDH formulation can readily accommodate partially labelled data by
incorporating the
linear constraints associated with the labelled data into the clustering formulation.
Without loss of generality assume that the first~$\ell$ examples are labelled
by $\y = (y_1,\ldots,y_\ell)^{\top} \in \{-1,1\}^\ell$.
The MDH for semi-supervised classification is the solution to the
problem,
\begin{subequations}\label{eq:SSconstr1}
\begin{align}
& \min_{(\v, b) \in  \B \times \R} \  \hat I(\v, b), \label{eq:SSLobj1} \\
\textrm{subject to:} & \quad  \ y_i(\v \cdot \x_i - b) \geqslant 0, \
	\forall i=1,\ldots,\ell, \label{eq:SSLconstr} \\
	&\quad b - \mu_{\v} + \alpha \sigma_\v \geqslant 0,
	\label{eq:SSLConstr2}\\
&\quad \mu_\v + \alpha \sigma_\v - b \geqslant 0,\label{eq:SSLConstr3}
\end{align}
\end{subequations}
where $\hat I(\v, b), \mu_\v$, and $\sigma_\v$ are computed over the entire
data set.
If the labelled examples are linearly separable the constraints in
Equation~(\ref{eq:SSconstr1}) define a nonempty feasible set of hyperplanes,
\begin{equation}
F_{\mathrm{LB}} = \left\{ H(\v,b) \,|\, (\v,b) \in \B\times\R, b \in F(\v), y_i(\v \cdot \x_i - b) \geqslant 0, \  \forall i \in \{1,\ldots,l\} \right\}\subset C.	
\end{equation}
Equations~(\ref{eq:SSLConstr2}) and~(\ref{eq:SSLConstr3}) act as a {\em balancing
constraint} which discourages MDHs that classify the vast majority of unlabelled data to
a single class. Balancing constraints are included in the estimation of
\S3VMs for the same reason~\citep{Joachims1999,ChapelleZ2005}.

As in the case of clustering, the direct minimisation of Equation~(\ref{eq:SSconstr1})
frequently leads to locally optimal solutions.
To mitigate this we again propose a projection pursuit formulation. 
We define the penalised density integral for semi-supervised classification as,
\begin{align}\label{eq:PenalisedSSL}
\fSL(\v,b) = \fUL(\v,b) + \gamma \sum_{i=1}^l \max \left\{0, -y_i (\v \cdot \x_i - b)\right\}^{1+\epsilon}
\end{align}
where, $\gamma>0$ is a user-defined constant, which controls the trade-off
between reducing the density on the hyperplane, and misclassifying the
labelled examples.
The projection index is then defined as the minimum of the penalised density integral,
\begin{align}
\phiSSL(\v) & = \min_{b \in \R} \fSL(\v,b) .\label{eq:PrIndexSSL} 
\end{align}

\section{Connection to Maximum Margin Hyperplanes}
\label{sec:MaxMarg}

In this section we discuss the connection between MDHs and maximum (hard) margin hyperplane separators.
The margin of a hyperplane $H(\mathbf{v}, b)$ with respect to a data set~$\X$
is defined as the minimum Euclidean distance between the hyperplane and its nearest datum,
\begin{equation}
\mbox{margin}\, H(\mathbf{v}, b) = \min_{\x \in \X}\vert \mathbf{v}\cdot \x - b \vert.
\end{equation}
The points whose distance to the hyperplane $H(\v,b)$ is equal to the margin of the hyperplane,
that is,
$
\argmin_{\x \in \X} \vert \mathbf{v}\cdot \x - b \vert,
$
are called the {\em support points} of $H(\v,b)$.
Let $F$ denote the set of feasible hyperplanes;
then the {\em maximum margin hyperplane} (MMH), $H(\v^m,b^m) \in F$ satisfies,
\begin{equation}
\margin\, H(\v^m,b^m) = \max_{(\v, b) \vert H(\v, b)\in F} \margin\, H(\v,b).
\end{equation}

The main result of this section is Theorem~\ref{thm:convergence},
which states that as the bandwidth parameter, $h$, is reduced to zero the
MDH converges to the MMH.
An intermediate result, Lemma~\ref{lm:samedivision}, shows that there exists a
positive bandwidth, $h'>0$ such that, for all $h \in (0,h')$, the partition of
the data set induced by the MDH is identical to that
of the MMH.

We first discuss some assumptions which allow us to present the
theoretical results of this section. As before we assume a fixed and
finite data set $\X \subset \R^d$, and approximate its (assumed) underlying
probability density function via a kernel density estimator using Gaussian
kernels with isotropic bandwidth matrix $h^2 I$. We assume that the interior of
the convex hull of the data, $\mbox{Int}(\conv \X)$, is non-empty,  and define
$C$ as the set of hyperplanes that intersect~$\mbox{Int}(\conv \ \X)$, as in
Equation~(\ref{eq:domain}). The set of feasible hyperplanes, $F$, for either
clustering or the semi-supervised classification satisfies $F \subset C$. By
construction every $H(\mathbf{v}, b) \in F$ defines a hyperplane which
partitions $\X$ into two non-empty subsets.
Observe that if for each $\v \in \B$ the set $\{b \in \R \,\vert\, H(\v, b) \in
F\}$ is compact, then by the compactness of $\B$ a maximum margin hyperplane
in~$F$ exists. For both the clustering and semi-supervised classification
problems this compactness holds by construction.

For any $h>0$, let $(\v^\star_h, b^\star_h) \in \B\times\R$ parameterise a hyperplane which
achieves the minimal density integral over all hyperplanes in $F$, for
bandwidth matrix $h^2I$. That is,
\begin{equation}
\hat{I}(\v^\star_h, b^\star_h) = \min_{(\v, b) \vert H(\v, b)\in F} \hat I(\v, b).
\end{equation}
Following the approach of~\citet{TongK2000} we first show that as the
bandwidth,~$h$, is reduced towards zero, the density on a hyperplane is
dominated by its nearest point.
This is achieved by establishing that for all sufficiently small values of $h$,
a hyperplane with non-zero margin has lower density integral than
any other hyperplane with smaller margin.

\begin{lemma}\label{lm:bddmargin}
Take $H(\v,b) \in F$ with non-zero margin and $0< \delta < \margin\, H(\v,b): = M_{\v, b}$.
Then $\exists h' >0$ such that $h \in (0,h')$ and $M_{\mathbf{w}, c}:=\margin\, H(\mathbf{w}, c) \leqslant M_{\v, b}-\delta$
implies $\hat{I}(\v,b) < \hat{I}(\w,c)$.
\end{lemma}

\begin{proof}

Using Equation~(\ref{eq4}) it is easy to see that,
\begin{align*}
\hat{I}(\v,b) & \leqslant \frac{1}{ h \sqrt{2 \pi} } \exp \left\{ - \frac{M_{\v, b}^2}{2 h^2} \right\}, \\
\inf \left\{ \hat I(\w,c) \, | \, M_{\w, c} \leqslant M_{\v, b} - \delta \right\} & \geqslant
	\frac{1}{n  h \sqrt{2 \pi} } \exp  \left\{ - \frac{(M_{\v, b}-\delta)^2}{2 h^2} \right\}.
\end{align*}
Therefore,
\begin{align*}
0 \leqslant \lim_{h \to 0^+} \frac{ \hat I(\v,b) }{\inf \left\{ \hat I(\w,c) \, | \, M_{\w, c} \leqslant M_{\v, b} - \delta \right\}}
& \leqslant \lim_{h \to 0^+} \frac{ n \exp \left\{ - \frac{M_{\v, b}^2}{2 h^2} \right\} }{ \exp \left\{ -\frac{(M_{\v, b}-\delta)^2}{2 h^2} \right\} } = 0.
\end{align*}

Therefore, $\exists h'>0$ such that $h \in (0,h^\prime) \Rightarrow \frac{\hat I(\v, b)}{\inf\left\{\hat I(\w, c) \big \vert M_{\w, c} \leqslant M_{\v, b}-\delta \right\}}<1$.

\end{proof}

An immediate corollary of Lemma~\ref{lm:bddmargin} is that as~$h$ tends to zero the margin of
the MDH tends to the maximum margin. However, this does
not necessarily ensure the stronger result that the sequence of MDHs
converges to the MMH. To establish this we require
two technical results, which describe some algebraic properties of
the MMH, and are provided as part of the proof of Theorem~\ref{thm:convergence}
which is given in Appendix~\ref{app:thm:convergence}.


The next lemma uses the previous result to show that there exists a
positive bandwidth, $h^\prime>0$, such that an MDH estimated using
$h\in (0,h^\prime)$ induces the same partition of $\X$ as the MMH.  The
result assumes that the MMH is unique. Notice
that if $\X$ is a sample of realisations of a continuous random variable then
this uniqueness holds with probability~1.

\begin{lemma} \label{lm:samedivision}

Suppose there is a unique hyperplane in $F$ with maximum margin, which can be
parameterised by $(\v^m, b^m) \in \B \times \R$.
Then $\exists h'>0$ s.t. $h \in (0,h^\prime) \Rightarrow H(\v^\star_h, b^\star_h)$ induces the same partition
of $\X$ as $H(\v^m, b^m)$.

\end{lemma}

\begin{proof}

Let $M = \margin\, H(\v^m, b^m)$, and let $P$ be the collection of
hyperplanes that induce the same partition of $\X$ as that induced by
$H(\v^m, b^m)$. 
Since $\X$ is finite and $H(\v^m,b^m)$ is unique, $\exists \delta > 0$ s.t. $H(w, c) \notin P
\Rightarrow \margin\, H(w, c) \leqslant M-\delta$.
By Lemma~\ref{lm:bddmargin}, $\exists h'>0$ s.t., \[h\in (0,h^\prime)
\Rightarrow H(\v^{\star}_h, b^{\star}_h) \notin \left\{H(\w, c) \, \vert \, \margin\, H(\w, c)
\leqslant M-\delta\right\},\] 
therefore $H(\v^{\star}_h, b^{\star}_h) \in P$.

\end{proof}

\noindent
The next theorem is the main result of this section, and states that the MDH converges to the MMH as the
bandwidth parameter is reduced to zero.  Notice that by
the non-unique representation of
hyperplanes, the maximum margin hyperplane has two parameterisations in $C$, namely $(\mathbf{v}^m,
b^m)$ and $(-\mathbf{v}^m, -b^m)$. Convergence to the
maximum margin hyperplane is therefore equivalent to showing that,
\[
\min\{\|(\mathbf{v}_h^\star, b_h^\star)-(\mathbf{v}^m, b^m)\|, \|(\mathbf{v}_h^\star, b_h^\star)+(\mathbf{v}^m, b^m)\|\} \to 0 \mbox{ as } h \to 0^+.
\]

\begin{theorem}\label{thm:convergence}

Suppose there is a unique hyperplane in $F$ with maximum margin, which can
be parameterised by $(\v^m, b^m) \in \B \times \R$.
Then,
\[
\lim_{h \to 0^+} \min \left\{\|(\v^\star_h, b^\star_h) - (\v^m, b^m)\|,
	\|(\v^\star_h, b^\star_h) + (\v^m, b^m)\| \right\} = 0.
\]

\end{theorem}

The set~$F$ used in Theorem~\ref{thm:convergence} is generic so it can capture
the constraints associated with both clustering and semi-supervised
classification, Equations~(\ref{eq:ULconstr1}) and~(\ref{eq:SSconstr1})
respectively. In the case of semi-supervised classification we must also assume
that the labelled data are linearly separable.
Theorem~\ref{thm:convergence} is not directly applicable to the \MDP
formulations as in this case the function being minimised is not the density on
a hyperplane.
The next two subsections establish this result for the \MDP formulation of the
clustering and semi-supervised classification problem.

\subsection{\MDP for Clustering}

We have shown that for the constrained optimisation formulation the MDH
converges to the MMH within the feasible set,
$F_{\mathrm{CL}} \subset C$. 
In addition, for a fixed $\v$, Proposition~\ref{prop:minim}
bounds the distance between minimisers of the penalised
function $\fUL$, $\argmin_{b \in \R} \fUL(\v,b)$,  and the optimal $b$ of the constrained
problem,~$\argmin_{b \in F(\v)} \hat{I}(\v,b)$.
Combining these we can show that the optimal solution to
the penalised MDP$^2$ formulation converges to the maximum margin hyperplane in
$F_{\mathrm{CL}}$, provided the parameters within the penalty term suitably
depend on the bandwidth parameter, $h$. While the general case can be shown,
for ease of exposition we make the simplifying assumption that the maximum
margin hyperplane is strictly feasible, i.e., if $(\v^m, b^m)$ parameterises
the maximum margin hyperplane then $b^m \in (\mu_{\v^m}-\alpha \sigma_{\v^m},
\mu_{\v^m}+\alpha \sigma_{\v^m})$.

For $h, \eta, L > 0$ define $(\v^\star_{h, \eta, L}, b^\star_{h, \eta, L})$ to be any global minimiser of $f_{\mathrm{CL}}$, i.e.,
$$
f_{\mathrm{CL}}(\v^\star_{h, \eta, L}, b^\star_{h, \eta, L}) = \min_{(\v, b) \in\B\times\R} f_{\mathrm{CL}}(\v, b).
$$

\begin{lemma}\label{lem:ULConv}
Suppose there is a unique hyperplane in $F_{\mathrm{CL}}$ with maximum margin, which can
be parameterised by $(\v^m, b^m) \in \B \times \R$. Suppose further that $b^m \in (\mu_{\v^m}-\alpha \sigma_{\v^m}, \mu_{\v^m}+\alpha \sigma_{\v^m})$.
For $h>0$, let
$L(h) = (e^{1/2}h^2\sqrt{2\pi})^{-1}$, and $0 < \eta(h) \leqslant h$.
Then,
$$
\lim_{h \to 0^+} \min\{\|(\v^\star_{h, \eta(h), L(h)}, b^\star_{h, \eta(h), L(h)}) - (\v^m,b^m)\|, \|(\v^\star_{h, \eta(h), L(h)}, b^\star_{h, \eta(h), L(h)}) + (\v^m, b^m)\|\} = 0.
$$
\end{lemma}

\begin{proof}

Let $M = \mbox{margin}H(\v^m, b^m)$ and as in the proof of Lemma \ref{lm:samedivision}, let $\delta > 0$ be such that
any hyperplane inducing a different partition from $H(\v^m, b^m)$ has margin at most $M-\delta$.
Consider the set $F^\delta_{\mathrm{CL}}: = \{(\v, b) \in \B \times \R \vert b \in \mathbb{B}_{\delta/2}(F(\v))\}$,
where we used the notation $\mathbb{B}_{\delta/2}(F(\v))$ to denote the neighbourhood of $F(\v)$ given by $\{r \in \R \vert d(r,F(\v))<\delta/2\}$.
The set $F^\delta_{\mathrm{CL}}$
increases the feasible set of hyperplanes by allowing $b$ to range in $b \in \mathbb{B}_{\delta/2}(F(\v))$.
For any fixed $\v$, the maximum margin of all hyperplanes with normal vector $\v$ can increase by at most $\delta/2$.
Thus, any hyperplane inducing a different partition compared to $H(\v^m,b^m)$ has a margin at most $M-\delta/2$.
Since $H(\v_m, b_m)$ is strictly feasible it therefore remains the unique maximum margin hyperplane in $F^\delta_{\mathrm{CL}}$.
Observe now that for $0<h<\delta/2$ we have $H(\v^\star_{h,
\eta(h), L(h)}, b^\star_{h, \eta(h), L(h)}) \in F_{\mathrm{CL}}^\delta$, by
Proposition~\ref{prop:minim}. In addition, by Theorem~\ref{thm:convergence}, we
know that the minimisers of $\hat I(\v, b)$ over $F^\delta_{\mathrm{CL}}$, say
$H(\v_h^\delta, b_h^\delta)$, satisfy
\[
\lim_{h\to 0^+} \min \left\{\|(\v_h^\delta, b_h^\delta) - (\v^m, b^m)\|, \|(\v_h^\delta, b_h^\delta) + (\v^m, b^m)\|\right\} = 0.
\]
Now, since $H(\v^m, b^m)$ is strictly feasible $\exists \epsilon^\prime > 0$
s.t. $(\v, b) \in \mathbb{B}_{\epsilon^\prime}(\{(\v^m, b^m), -(\v^m, b^m)\})
\Rightarrow H(\v, b) \in F_{\mathrm{CL}}$. Then for any $0 < \epsilon <
\epsilon^\prime$ there exists $h^\prime > 0$ s.t. for $0<h<h^\prime$ both
$(\v_h^\delta, b_h^\delta) \in \mathbb{B}_{\epsilon}(\{(\v^m, b^m), -(\v^m,
b^m)\}) \Rightarrow H(\v_h^\delta, b_h^\delta) \in F_{\mathrm{CL}}$ \emph{and}
$H(\v^\star_{h, \eta(h), L(h)}, b^\star_{h, \eta(h), L(h)}) \in
F_{\mathrm{CL}}^\delta$.  Now for $H(\v, b) \in F^\delta_{\mathrm{CL}}\setminus
F_{\mathrm{CL}}$ we know that $\hat I(\v, b) < f_{\mathrm{CL}}(\v, b)$, whereas
for $H(\v, b) \in F_{\mathrm{CL}}, \hat I(\v, b) = f_{\mathrm{CL}}(\v, b)$ and
therefore the minimiser of $f_{\mathrm{CL}}(\v, b)$ must lie in the
neighbourhood $\mathbb{B}_{\epsilon}(\{(\v^m, b^m), -(\v^m, b^m)\})$, and the
result follows.

\end{proof}

To illustrate the convergence of the MDH to the MMH we use the two-dimensional
data set shown in Figure~\ref{fig:MMHCon}. The data is sampled from a mixture
of two Gaussian distributions with equal covariance matrix. 
The MDH with respect to the true underlying density is~$H\left( (1,-1),0
\right)$.
A large margin separator is artificially introduced by removing a few
observations in a narrow margin around a hyperplane different from~$H\left(
(1,-1),0 \right)$.  The margin is intentionally small to ensure that
identifying the MMH is non-trivial.  Figure~\ref{fig:MMHCon} illustrates the
MDH solutions arising from the \MDP method for a decreasing sequence of
bandwidths,~$h$.  Initially the MDH approximately coincides with the optimal
MDH with respect to the true density of the Gaussian mixture.  As~$h$
decreases, the MDH approaches the MMH and for the smallest values of $h$ the
two are indistinguishable.
\begin{figure}\centering
\includegraphics[scale=0.25]{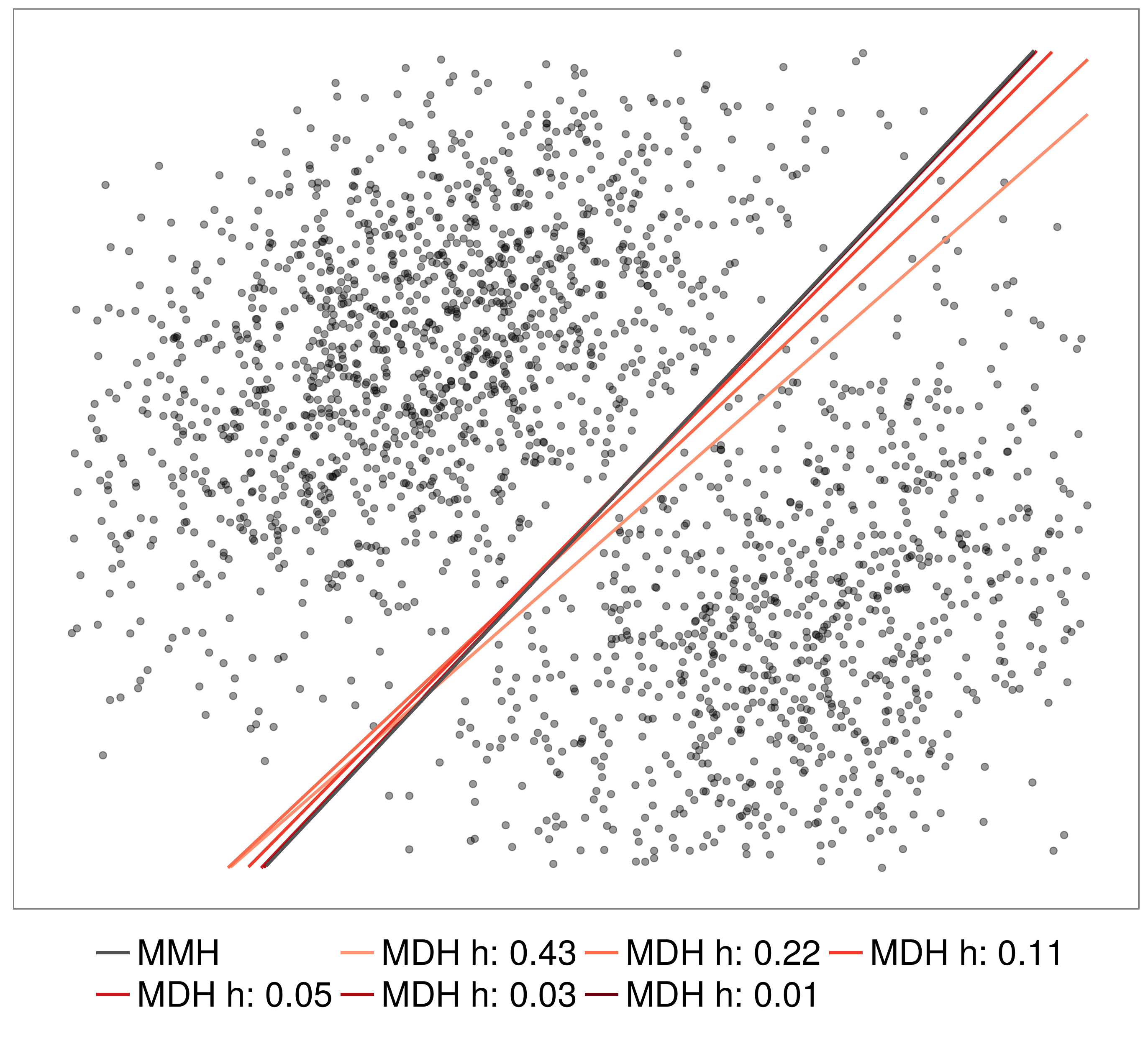}
\caption{Convergence of the MDH to the maximum margin hyperplane for a decreasing sequence of
bandwidth parameters, $h$.}\label{fig:MMHCon}
\end{figure}

\subsection{\MDP for Semi-Supervised Classification}

Denote the set of hyperplanes which correctly classify the
labelled data by $F_{\mathrm{LB}}$. Under the assumption that
$\exists H(\v, b) \in F_{\mathrm{LB}}\cap F_{\mathrm{CL}}$
with non-zero margin, we
can show that, provided the parameter $\gamma$ does not shrink too quickly with
$h$, the hyperplane that minimises~$\fSL$ converges to
the MMH contained in~$F_{\mathrm{LB}} \cap
F_{\mathrm{CL}}$, where as before we assume that such an MMH
is strictly feasible.
To establish this result it is sufficient to show that there exists $h'>0$ such
that for all $h \in (0,h^{\prime})$, the optimal hyperplane $H(\v^\star_{h,\eta,L,\gamma}, b^\star_{h, \eta,L,
\gamma})$ correctly classifies all the labelled examples.
If this holds, then $\fSL(\v^\star_{h,\eta,L,\gamma}, b^\star_{h,\eta,L,\gamma}) = \fUL(\v^\star_{h,\eta,L,\gamma}, b^\star_{h,\eta,L,\gamma})$ for
all sufficiently small~$h$, and hence
Lemma~\ref{lem:ULConv} can be applied to establish the
result.
The proof relies on the fact that the penalty terms
associated with the known labels in Equation~(\ref{eq:PenalisedSSL}) are polynomials
in~$b$. Provided that $\gamma$ is bounded below by a polynomial in $h$, the value of the
penalty terms for hyperplanes that do not correctly classify the labelled data dominate the value of the density integral as $h$ approaches zero. Therefore the optimal hyperplane
must correctly classify the labelled data for small values of $h$.

\begin{lemma}
Define $F_{\mathrm{LB}} = \{H(\v, b) \big \vert y_i(\v\cdot x_i - b) > 0, \forall
i=1, \dots, \ell\}$ and $F_{\mathrm{CL}} = \{H(\v, b) \big \vert \mu_{\v} -
\alpha\sigma_{\v} \leqslant b \leqslant \mu_{\v} + \alpha\sigma_{\v}\}$ and assume
that $F_{\mathrm{SSC}} = F_{\mathrm{LB}} \cap F_{\mathrm{CL}} \not = \emptyset$ and that $\exists H(\v, b) \in 
F_{\mathrm{SSC}}$ with non-zero margin.
For $h > 0$, let $L(h) = (e^{1/2}h^2\sqrt{2\pi})^{-1}$, $0 < \eta(h) \leqslant h$ and $\gamma(h) \geqslant h^r$ for some $r>0$.
Then $\exists h^\prime > 0$ s.t. $h \in(0,h^\prime) \Rightarrow H(\v^\star_{h, \eta(h), L(h), \gamma(h)},
b^\star_{h, \eta(h), L(h), \gamma(h)} ) \in F_{\mathrm{LB}}$.
\end{lemma}

\begin{proof}

Consider $H(\v, b) \not \in F_{\mathrm{LB}}$. Then,
$$\fSL(\v, b) \geqslant \frac{1}{n\sqrt{2\pi}h}\exp(-\nu_\star^2/2h^2) + \gamma(h)\nu_\star^{1+\epsilon} > \gamma(h)\nu_\star^{1+\epsilon},$$
where $\nu_\star > 0$ minimises $\frac{1}{n\sqrt{2\pi}h}\exp(-\nu^2/2h^2) + \gamma(h)\nu^{1+\epsilon}$.
Therefore, $\nu_\star$ is the unique positive number satisfying,
\begin{align*}
\frac{1}{n\sqrt{2\pi}h}\exp\left(-\frac{\nu_\star^2}{2h^2}\right)\left(-\frac{\nu_\star}{h^2}\right) &+ (1+\epsilon)\gamma(h)\nu_\star^\epsilon = 0\\ &\Rightarrow \nu_\star^{1-\epsilon} = (1+\epsilon)\gamma(h)n\sqrt{2\pi}h^3\exp\left(\frac{\nu_\star^2}{2h^2}\right)\\
&\Rightarrow \nu_\star \geqslant \left((1+\epsilon)\gamma(h)n\sqrt{2\pi}h^3\right)^{1/1-\epsilon}.
\end{align*}
We therefore have,
\begin{eqnarray*}
\fSL(\v, b) &>& \gamma(h)\left((1+\epsilon)\gamma(h)n\sqrt{2\pi}h^3\right)^{\frac{1+\epsilon}{1-\epsilon}}\\
&=& K \gamma(h)^{\frac{2}{1-\epsilon}}h^{\frac{3(1+\epsilon)}{1-\epsilon}}\\
&\geqslant& K h^{\frac{2r+3(1+\epsilon)}{1-\epsilon}},
\end{eqnarray*}
where $K$ is a constant which can be chosen independent of $(\v, b)$. Finally,
for any $H(\v^\prime, b^\prime) \in F_{\mathrm{SSC}}$ with non-zero margin,
$\exists h^\prime > 0$ s.t.  
\[ h \in(0,h^\prime) \Rightarrow \fSL(\v^\prime,
b^\prime) = \hat I(\v^\prime, b^\prime) < K
h^{\frac{2r+3(1+\epsilon)}{1-\epsilon}} < \fSL(\v, b).\]
Since $K$ is
independent of $(\v, b)$, the result follows. The final set of inequalities
holds since the hyperplane $H(\v^\prime, b^\prime)$ is assumed to have non-zero
margin, say $M_{\v^\prime, b^\prime} > 0$, and hence $\hat I(\v^\prime, b^\prime) \leqslant \frac{1}{h\sqrt{2\pi}}\exp\{-M_{\v^\prime, b^\prime}/2h^2\}$, which tends to zero faster
than any polynomial in $h$.

\end{proof}

\section{Estimation of Minimum Density Hyperplanes}
\label{sec:Methodology}

In this section we discuss the computation of MDHs.
We first investigate the continuity and differentiability properties
required to optimise the projection indices~$\phiUL(\v)$ and~$\phiSSL(\v)$.

Since the domain of both projection indices, $\phiUL(\v)$ and~$\phiSSL(\v)$,
is the boundary of the unit-sphere in~$\R^d$
it is more convenient
to express $\v$ in terms of spherical coordinates,
\begin{equation}\label{eq:Spherical}
v_i(\theta)  = \left\{\begin{array}{ll}\cos(\theta_i) \prod_{j=1}^{i-1}
\sin(\theta_j),& i=1,\ldots,d-1  \\
\prod_{j=1}^{d-1} \sin(\theta_j), & i = d, \end{array}\right.
\end{equation}
where $\theta \in \Th = [0, \pi]^{d-2} \times [0,2\pi]$ is called the {\em projection angle}\/. 
Using spherical coordinates renders the domain, $\Th$, convex and compact, and
reduces dimensionality by one.

As the following discussion applies to both~$\phiUL(\v)$ and~$\phiSSL(\v)$
we denote a generic projection index $\phi:\Theta \to \R$, and the associated set of minimisers, as,
\begin{align}\label{OpVal}
\phi(\theta) &= \min_{b \in A} f(\v(\theta),b),\\
\hY(\theta) &= \left\{ b \in A\, \big| \, f(\v(\theta), b) = \phi(\theta) \right\},
\end{align}
where $f(\v(\theta), b)$ is continuously differentiable,
$A \subset \R$ is compact and convex, and the correspondence
$\hY(\theta)$
gives the set of global minimisers of $f(\v(\theta),b)$ for each $\theta$. 
The definition of~$A$ is not critical in our formulation. 
Setting,
\begin{equation}
A \supset \left[\min_{\v \in \B}\{\mu_\v\}-\alpha \sigma_{\textrm{pc}_1} -\eta, \max_{\v \in \B}\{\mu_\v\}+\alpha \sigma_{\textrm{pc}_1} + \eta \right],
\end{equation}
where $\sigma^2_{\textrm{pc}_1}$ is the variance of the projections
along the first principal component, ensures that the set of hyperplanes that satisfy
the constraint of Equation~(\ref{constr1}) will be a subset of $A$ for all $\v$.

Berge's maximum theorem~\citep{Berge1963,Polak1987}, establishes the continuity
of $\phi(\theta)$ and the upper-semicontinuity (u.s.c.) of the correspondence
$\hY(\theta)$.
Theorem~3.1 in~\citep{Polak1987} enables us to establish that $\phi(\theta)$ is
locally Lipschitz continuous. Using Theorem 4.13 of~\citet{BonnansS2000} we can
further show that $\phi(\theta)$ is directionally differentiable everywhere.
The directional derivative at $\theta$ in the direction $\nu$ is given by,
\begin{equation}\label{OptDif}
d \phi(\theta; \nu) = \min_{b \in \hY(\theta)} D_{\theta} f(\v(\theta), b)
\cdot \nu,
\end{equation}
where $D_{\theta}$ denotes the derivative with respect to $\theta$. 
It is clear from Equation~(\ref{OptDif}) that $\phi(\theta)$ is differentiable if
$D_{\theta} f(\v(\theta), b)$ is the same for all $b \in \hY(\theta)$. If
$B(\theta)$ is a singleton then this condition is trivially satisfied and
$\phi(\theta)$ is continuously differentiable at $\theta$.

It is possible to construct examples in which $B(\theta)$ is not a singleton.
However, with the exception of contrived examples, our experience with real and
simulated data sets indicates that when $h$ is set through standard bandwidth
selection rules $B(\theta)$ is almost always a singleton over the optimisation path.

\begin{proposition}
Suppose $B(\theta)$ is a singleton for almost all $\theta \in \Th$. Then $\phi(\theta)$ is continuously differentiable almost everywhere. 
\end{proposition}

\begin{proof}
The result follows immediately from the fact that if $B(\theta) = \{b\}$ is a singleton, then the derivative $D\phi(\theta) = D_\theta f(\v(\theta), b)$, which is continuous.
\end{proof}

\citet{Wolfe1972} has provided early examples of how standard gradient-based
methods can fail to converge to a local optimum when used to minimise nonsmooth
functions.
In the last decade a new class of nonsmooth optimisation algorithms has been
developed based on gradient sampling~\citep{BurkeLO2006}. Gradient sampling
methods use generalised gradient descent to find local minima. At each
iteration points are randomly sampled in a radius $\varepsilon$ of the current
candidate solution, and the gradient at each point is computed. The convex
hull of these gradients serves as an approximation of the $\varepsilon$-Clarke
generalised gradient~\citep{BurkeLO2002a}.
The minimum element in the convex hull of these gradients is a descent
direction. 
The gradient sampling algorithm progressively reduces the sampling radius so
that the convex hull approximates the Clarke generalised gradient.  When the
origin is contained in the Clarke generalised gradient there is no direction of
descent, and hence the current candidate solution is a local minimum.
Gradient sampling achieves almost sure global convergence for functions that
are locally Lipschitz continuous and almost everywhere continuously
differentiable.  It is also well documented that it is an effective
optimisation method for functions that are only locally Lipschitz continuous.

\subsection{Computational Complexity}

In this subsection we analyse the computational complexity of MDP$^2$.
At each iteration the algorithm projects the data sample onto $\v(\theta)$
which involves $\O(n d)$ operations.
To compute the projection index, $\phi(\theta)$, we need to minimise the
penalised density integral, $f(\v(\theta),b)$. This can be achieved by first
evaluating $f(\v(\theta),b)$ on a grid of $m$ points, to bracket the location
of the minimiser, and then applying bisection to compute the minimiser(s)
within the desired accuracy.
The main computational cost of this procedure is due to the first step which
involves $m$ evaluations of a kernel density estimator with $n$ kernels.  Using
the improved fast Gauss transform~\citep{Morariu08} this can be performed in
$\O(m+n)$ operations, instead of $\O(mn)$.  Bisection requires $\O(-\log_2
\varepsilon)$ iterations to locate the minimiser with accuracy~$\varepsilon$.

If the minimiser of the penalised density integral $b^\star = \argmin_{b \in A}
f(\v(\theta),b)$, is unique
the projection index is continuously differentiable at $\theta$. To obtain the
derivative of the projection index it is convenient to define the projection
function, 
$P(\v) = \left( \x_1 \cdot \v, \ldots, \x_n \cdot \v \right)^\top.$
An application of the chain rule yields,
\begin{equation}\label{Der}
d_\theta \phi = D_\theta f(\v(\theta), b^\star) = D_P f(\v(\theta), b^\star)
D_\v P D_\theta \v
\end{equation}
where the derivative of the projections of the data sample with respect to $\v$
is equal to the data matrix, $D_\v P = (\x_1, \ldots, \x_n)^{\top}$; and
$D_\theta \v$ is the derivative of $\v$ with respect to the projection angle,
which yields a $d\times (d-1)$ matrix.  The computation of the derivative
therefore requires $\O(d(n+d))$ operations.

The original GS algorithm requires $\O(d)$ gradient evaluations at each
iteration which is costly. \citet{CurtisQ2013} have developed an adaptive
gradient sampling algorithm that requires $\O(1)$ gradient evaluations in each
iteration. More recently, \citet{LewisO2013} have strongly advocated that for
the minimisation of nonsmooth, nonconvex, locally Lipschitz functions, a simple
BFGS method using inexact line searches is much more efficient in practice than
gradient sampling, although no convergence guarantees have been established for
this method. BFGS requires a single gradient evaluation at each iteration and a
matrix vector operation to update the Hessian matrix approximation. In our
experiments we use the BFGS algorithm.

\section{Experimental Results}\label{sec:Exper}

In this section we assess the empirical performance of MDHs
for clustering and semi-supervised
classification. We compare performance with existing state-of-the-art methods
for both problems on the following 14 benchmark data sets: Banknote
authentication (banknote), Breast Cancer Wisconsin original (br. cancer),
Forest type mapping (forest), Ionosphere, Optical recognition of handwritten
digits (optidigits), Pen-based recognition of hand-written digits (pendigits),
Seeds, Smartphone-Based Recognition of Human Activities and Postural
Transitions (smartphone), Statlog Image Segmentation (image seg.), Statlog
Landsat Satellite (satellite), Synthetic control chart time series (synth
control), Congressional voting records (voting), Wine, and Yeast cell cycle
analysis (yeast).  Details of these data sets, in terms of their size, $n$,
dimensionality, $d$ and number of clusters, $c$, can be seen in
Table~\ref{tb:datasets}.

\begin{table}
\begin{center}
\begin{tabular}{r|rrr}
		&		$n$ 	& $d$ 	& $c$ \\
\hline
banknote$^a$ & 1372 & 4 & 2\\
br. cancer$^a$ & 699 & 9 & 2\\
forest$^a$ & 523 & 27 & 4\\
ionosphere$^a$ & 351 & 33 & 2\\
optidigits$^a$ & 5618 & 64 & 10\\
pendigits$^a$ & 10992 & 16 & 10\\
seeds$^a$ & 210 & 7 & 3\\
smartphone $^a$ & 10929 & 561 & 12\\
image seg.$^a$ & 2309 & 18 & 7\\
satellite$^a$ & 6435 & 36 & 6\\
synth$^a$ & 600 & 60 & 6\\
voting$^a$ & 435 & 16 & 2\\
wine$^a$ & 178 & 13 & 3\\
yeast$^b$ & 698 & 72 & 5
\end{tabular}
\end{center}
{\scriptsize
a. UCI machine learning repository \url{https://archive.ics.uci.edu/ml/datasets.html}\\
b. Stanford Yeast Cell Cycle Analysis Project \url{http://genome-www.stanford.edu/cellcycle/}
}
\caption{Details of benchmark data sets: size~($n$), dimensionality~($d$), number of clusters~($c$).} 
\label{tb:datasets}
\end{table}

\subsection{Clustering}\label{ssec:expUL}

Since an MDH yields a bi-partition of a data set rather than a
complete clustering, we propose two measures to assess the quality of a binary partition
of a data set containing an arbitrary number of clusters.
Both take values in $[0, 1]$ with larger values indicating a better partition.
These measures are motivated by the fact that a good binary partition should
(a) avoid dividing clusters between elements of the partition, and (b) be able to
discriminate at least one cluster from the rest of the data.
To capture this we modify the cluster labels of the data by assigning each
cluster to the element of the binary partition which contains the
majority of its members. In the case of a tie the cluster is assigned to the
smaller of the two partitions. We thus merge the true clusters into two
aggregate clusters,~$C_{1}$ and $C_{2}$.

The first measure we use is the binary V-measure which is simply the
V-measure~\citep{Rosenberg2007} computed on $C_{1}, C_{2}$ with respect to the
binary partition, which we denote $\Pi_1, \Pi_2$.  The V-measure is the
harmonic mean of homogeneity and completeness. For a data set containing
clusters $C_1,\ldots,C_c$, partitioned as $\Pi_1,\ldots,\Pi_k$, homogeneity is
defined as the conditional entropy of the cluster distribution within each
partition, $\Pi_i$. Completeness is symmetric to homogeneity and measures the
conditional entropy of each partition within each cluster, $C_j$.
An important characteristic of the V-measure for evaluating binary partitions
is that if the distribution of clusters within each partition is equal to the
overall cluster distribution in the data set then the V-measure is equal to
zero~\citep{Rosenberg2007}.  This means that if an algorithm fails to
distinguish the majority of any of the clusters from the remainder of the data,
the binary V-measure returns zero performance.  Other evaluation metrics for
clustering, such as purity and the Rand index, can assign a high value to such
partitions.

To define the second performance measure we first determine the number of
correctly and incorrectly classified samples.
The error of a binary partition, $\mbox{E}(\Pi_1, \Pi_2)$, given in
Equation~(\ref{eq:error}), is defined as the number of elements of each aggregate
cluster which are not in the same partition as the majority of their original
clusters.
In contrast, the success of a partition, $\mbox{S}(\Pi_1, \Pi_2)$,
Equation~(\ref{eq:success}), measures the number of samples which are in the same
partition as the majority of their original clusters.
The Success Ratio, $\mbox{SR}(\Pi_1, \Pi_2)$, Equation~(\ref{eq:SR}), captures the
extent to which the majority of at least one cluster is well-distinguished from
the rest of the data.
\begin{align}
\mbox{E}(\Pi_1, \Pi_2) & = \min\left\{\vert \Pi_1 \cap C_1 \vert + \vert \Pi_2
	\cap C_2 \vert, \vert \Pi_1 \cap C_2 \vert + \vert \Pi_2 \cap C_1 \vert\right\}, \label{eq:error}\\
\mbox{S}(\Pi_1, \Pi_2) & = \min\left\{\max\left\{\vert\Pi_1 \cap C_1 \vert, \vert
\Pi_1 \cap C_2 \vert \right\}, \max\left\{\vert \Pi_2 \cap C_1 \vert, \vert \Pi_2 \cap C_2 \vert \right\}\right\}, \label{eq:success}\\
\mbox{SR}(\Pi_1, \Pi_2) & = \frac{ \mbox{S}(\Pi_1, \Pi_2) }{ \mbox{S}(\Pi_1, \Pi_2) + \mbox{E}(\Pi_1, \Pi_2)}. \label{eq:SR}
\end{align}
The Success Ratio takes the value zero if an
algorithm fails to distinguish the majority of any cluster from the remainder of
the data.

\subsubsection{Parameter Settings for \MDP}

The two most important settings for the performance of
the proposed approach are the initial projection direction, and the choice
of~$\alpha$, which controls the width of the interval $F(\v)$ within which the
optimal hyperplane falls.
Despite the ability of the \MDP formulation to
mitigate the effect of local minima of the projected density,
the problem remains non-convex and local minima in the projection index
can still lead to suboptimal performance. We have found that this effect is
amplified in general when either or both the number of dimensions, and
the number of high density clusters in the data set is large.
To better handle the effect of local optima, we use multiple initialisations
and select the MDH that maximises the
\textit{relative depth} criterion, defined in Equation~(\ref{eq:relDep}).
The relative depth of an MDH, $H(\v, b)$, is defined as the
smaller of the relative differences in the density on the MDH and its two adjacent
modes in the projected density,
\begin{equation}\label{eq:relDep}
\mbox{RelativeDepth}(\v, b) = \frac{ \min\left\{ \hat{I}(\v, m_l), \hat{I}(\v, m_r) \right\} - \hat{I}(\v, b)}{\hat{I}(\v, b)}
\end{equation}
where $m_l$ and $m_r$ are the two adjacent modes in 
the projected density on $\v$. If an MDH
does not separate the modes of the projected density, then
its relative depth is set to zero, signalling a failure of \MDP
to identify a meaningful bi-partition.
The relative depth is
appealing because it captures the fact that a high quality separating hyperplane
should have a low density integral, and separate well the
modes of the projected density.
Note also that the relative depth is equivalent to the inverse of a measure
used to define cluster overlap in the context of Gaussian mixtures~\citep{Aitnouri2000}.
In all the reported experiments we initialise \MDP to the
first and second principal component and select the MDH with the largest relative depth.
For the data sets listed above it was never the case that both initialisations
led to MDHs with zero relative depth.

The choice of~$\alpha$ determines the trade-off between a balanced bi-partition
and the ability to discover lower density hyperplanes. 
The difficulties associated with choosing this parameter are illustrated in
Figure~\ref{fig:alpha}.
In each sub-figure the horizontal axis is the candidate projection
vector,~$\v$, while the right vertical axis is the direction of maximum
variability orthogonal to $\v$. Points correspond to projections of
the data sample onto this two-dimensional space, while colour
indicates cluster membership. The solid line depicts the projected
density on $\v$, while the dotted line depicts the penalised
function, $\fUL(\v, \cdot)$. The scale of both functions is depicted on the left
vertical axis. The solid vertical line indicates the MDH along~$\v$.
Setting $\alpha$ to a large value can cause \MDP to focus on hyperplanes that
have low density because they partition only a small subset of the data set as shown
in Figure~\ref{fig:alpha1}.
In contrast smaller values of $\alpha$ may cause the algorithm
to disregard valid lower density hyperplane separators (see
Figure~\ref{fig:alpha2}), or for the separating hyperplane to not be a local
minimiser of the projected density (see Figure~\ref{fig:alpha3}).

\begin{figure}[!h]
\centering
\subfigure[MDH separating few observations]{\includegraphics[width = 5cm]{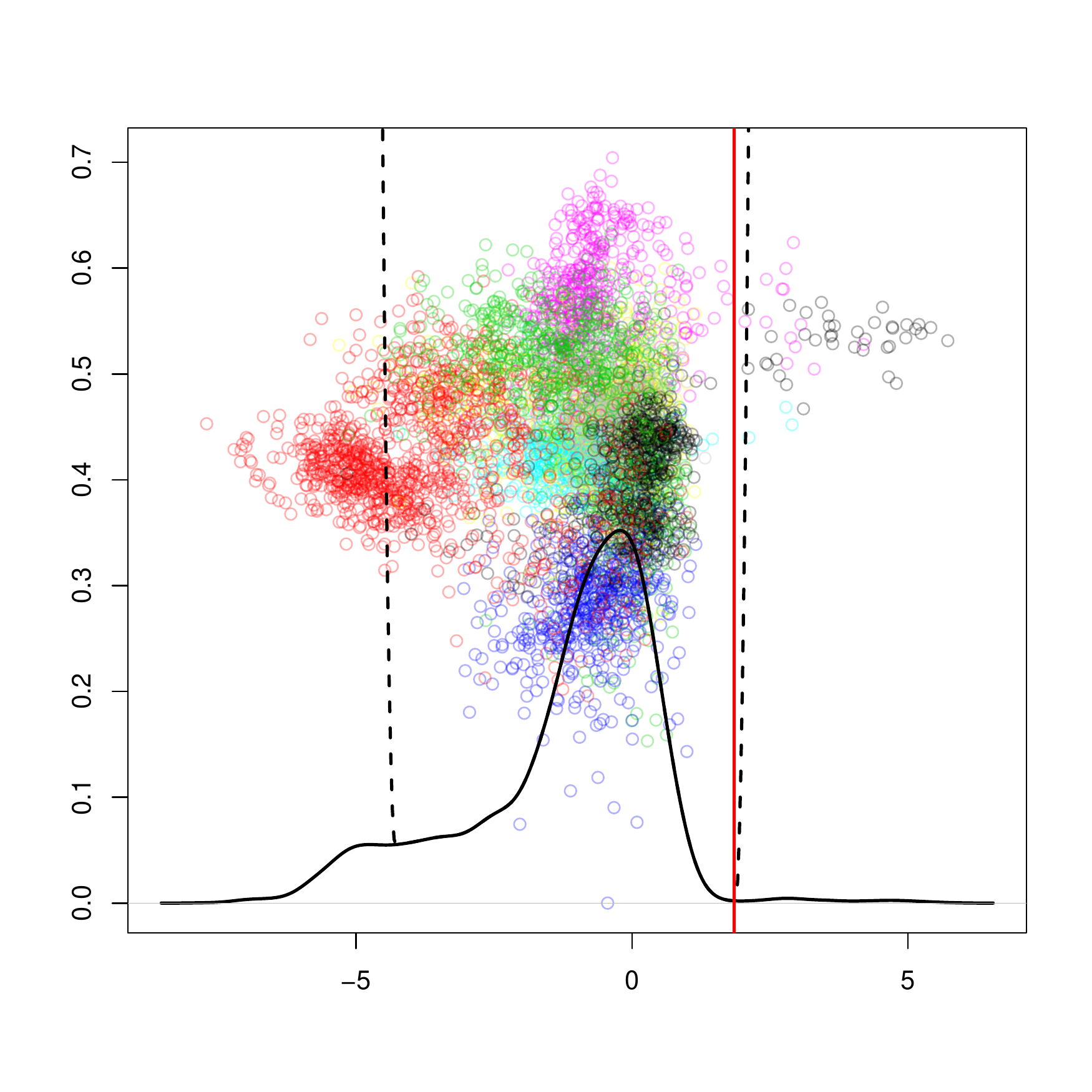}\label{fig:alpha1}}
\subfigure[Lower density hyperplane beyond feasible region]{\includegraphics[width = 5cm]{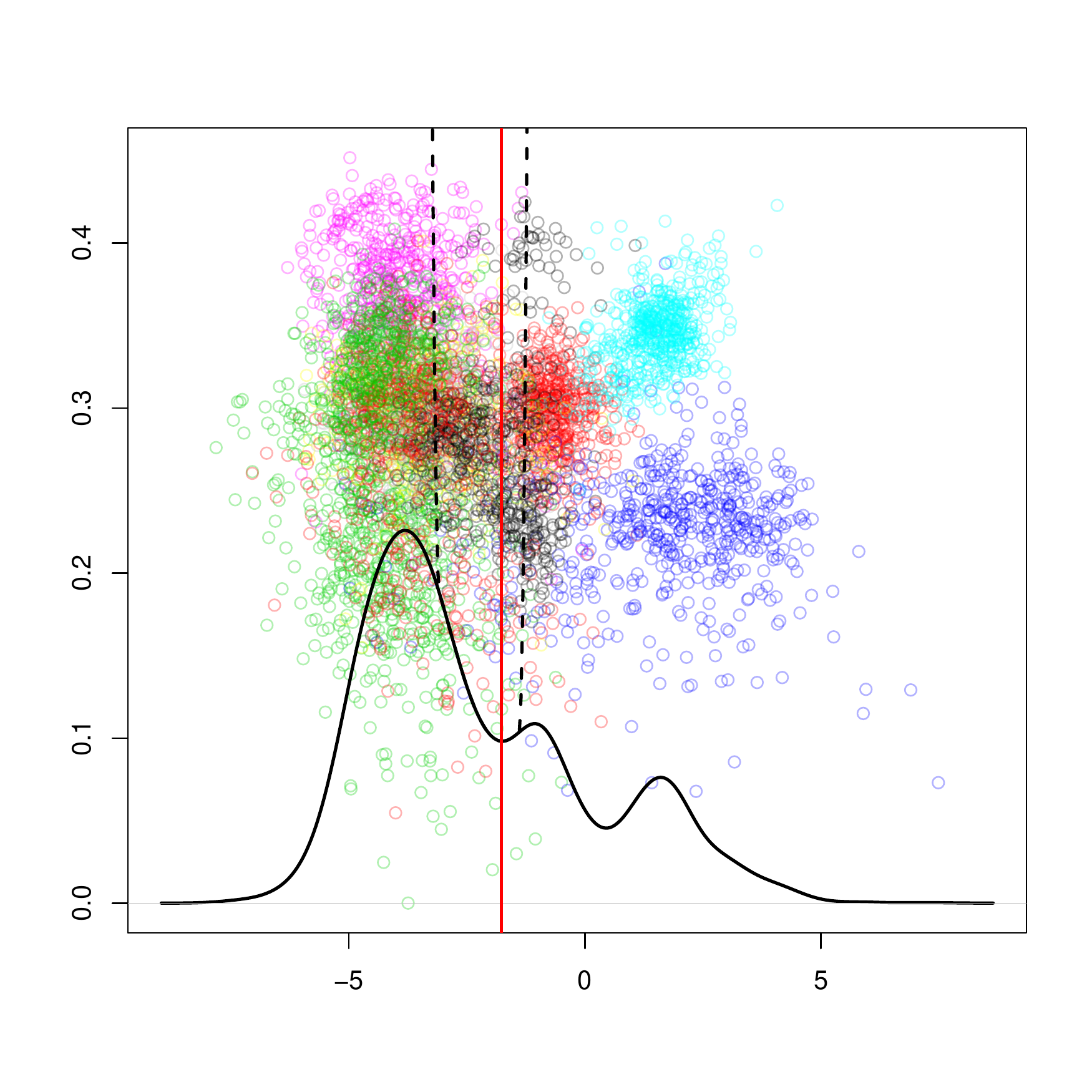}\label{fig:alpha2}}
\subfigure[MDH not a minimiser of the projected density]{\includegraphics[width = 5cm]{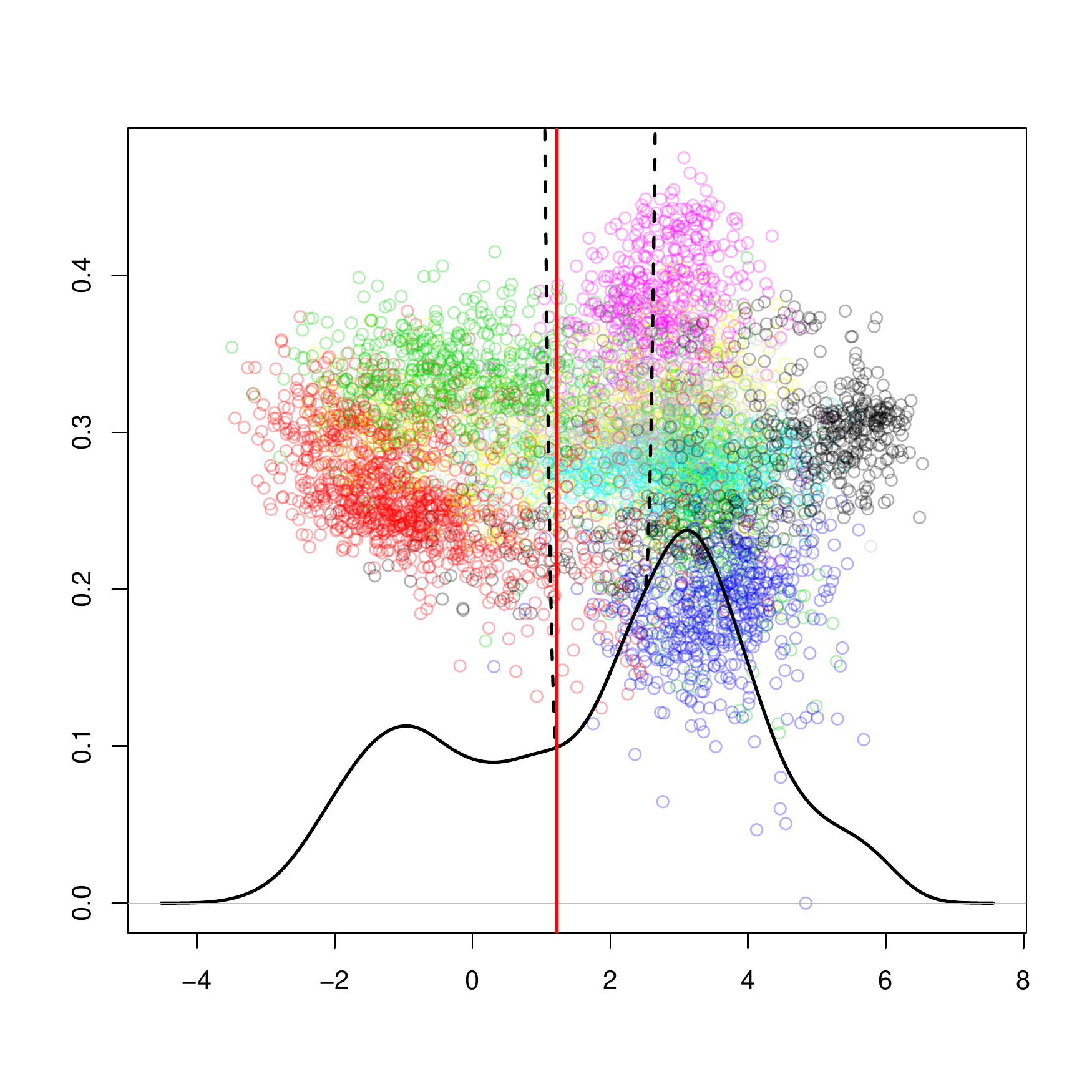}\label{fig:alpha3}}
\caption{Impact of choice of $\alpha$ on minimum density hyperplane.}\label{fig:alpha}
\end{figure}


Rather than selecting a single value for~$\alpha$ we recommend solving \MDP
repeatedly for an increasing sequence of values in the range $\{\alpha_{\min}, \alpha_{\max}\}$,
where each implementation beyond the first is initialised using the solution to the previous.
Setting
$\alpha_{\min}$ close to zero forces \MDP to seek low density hyperplanes that
induce a balanced data partition. This tends to find projections which display
strong multimodal structure, yet prevents convergence
to hyperplanes that have low density because they partition a few observations,
as in the case shown in Figure~\ref{fig:alpha1}.
Increasing $\alpha$ progressively fine-tunes the location of the MDH.
To avoid sensitivity to the value of $\alpha_{\max}$ (set to 0.9)
the output of the algorithm is the last hyperplane that corresponds to a
minimiser of the projected density.
Figure~\ref{fig:opti} illustrates this approach using the optical recognition
of handwritten digits data set from the UCI machine learning
repository~\citep{uci}. 
Figure~\ref{fig:opt1} depicts the projected density
on the initial projection direction, which in this case is the second
principal component. As shown, the density is unimodal and the clusters are not
well separated along this vector. Although not shown, if a large value of
$\alpha$ is used from the outset, \MDP will identify a vector along which 
the projected density is unimodal and skewed.
Figure~\ref{fig:opt2} shows that after five iterations with $\alpha=10^{-2}$
\MDP has identified a projection vector with bimodal density.
In subsequent iterations the two modes become more clearly separated,
Figure~\ref{fig:opt3},
while increasing $\alpha$ enables \MDP to locate an MDH
that corresponds to a minimiser of $\hat{I}(\v, b)$, as illustrated in Figure~\ref{fig:opt4}.

\begin{figure}[!ht]
\centering
\subfigure[Initial projection]{\includegraphics[width = 7cm, height = 6cm]{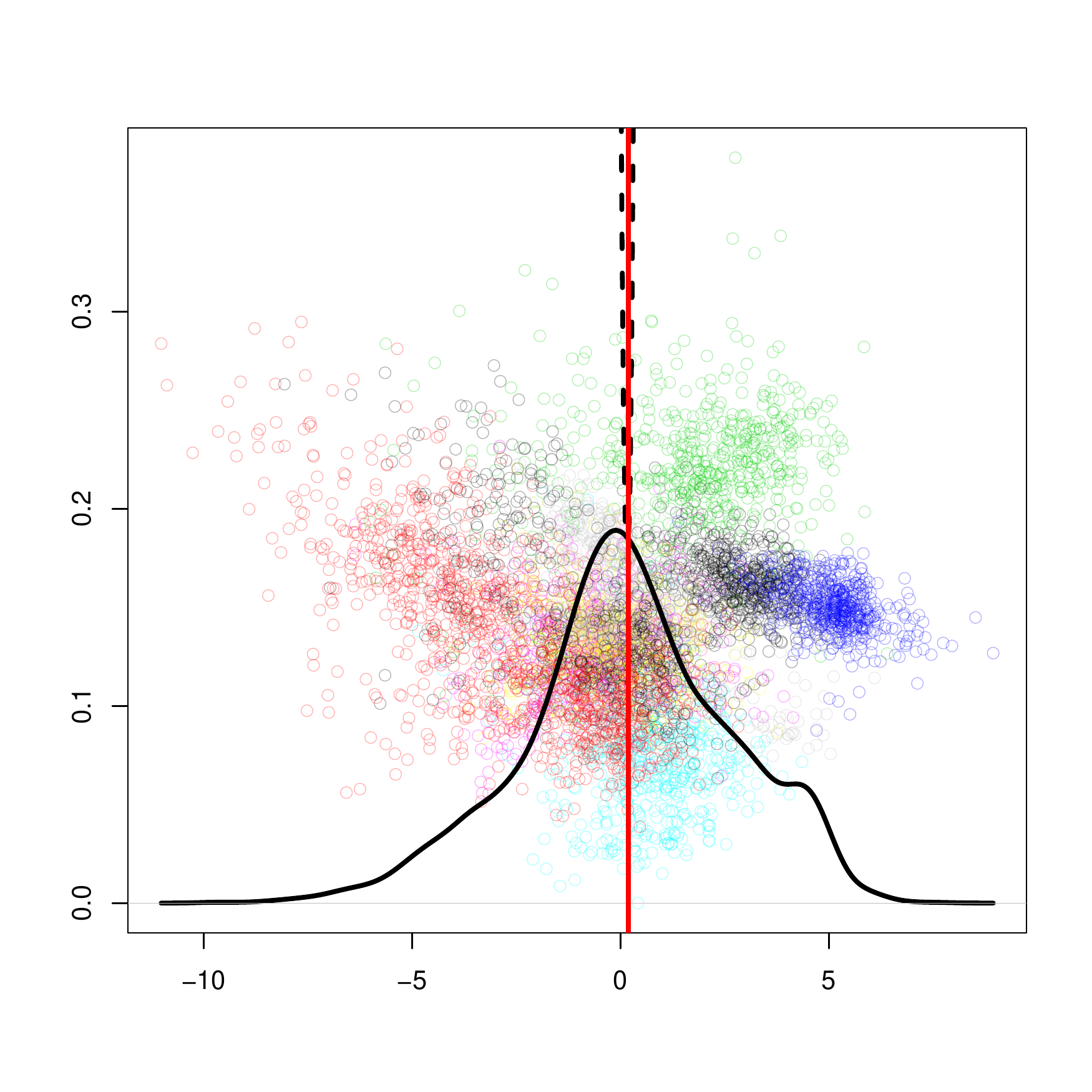}\label{fig:opt1}}
\subfigure[Projection after 5 iterations with $\alpha=10^{-2}$]{\includegraphics[width = 7cm, height = 6cm]{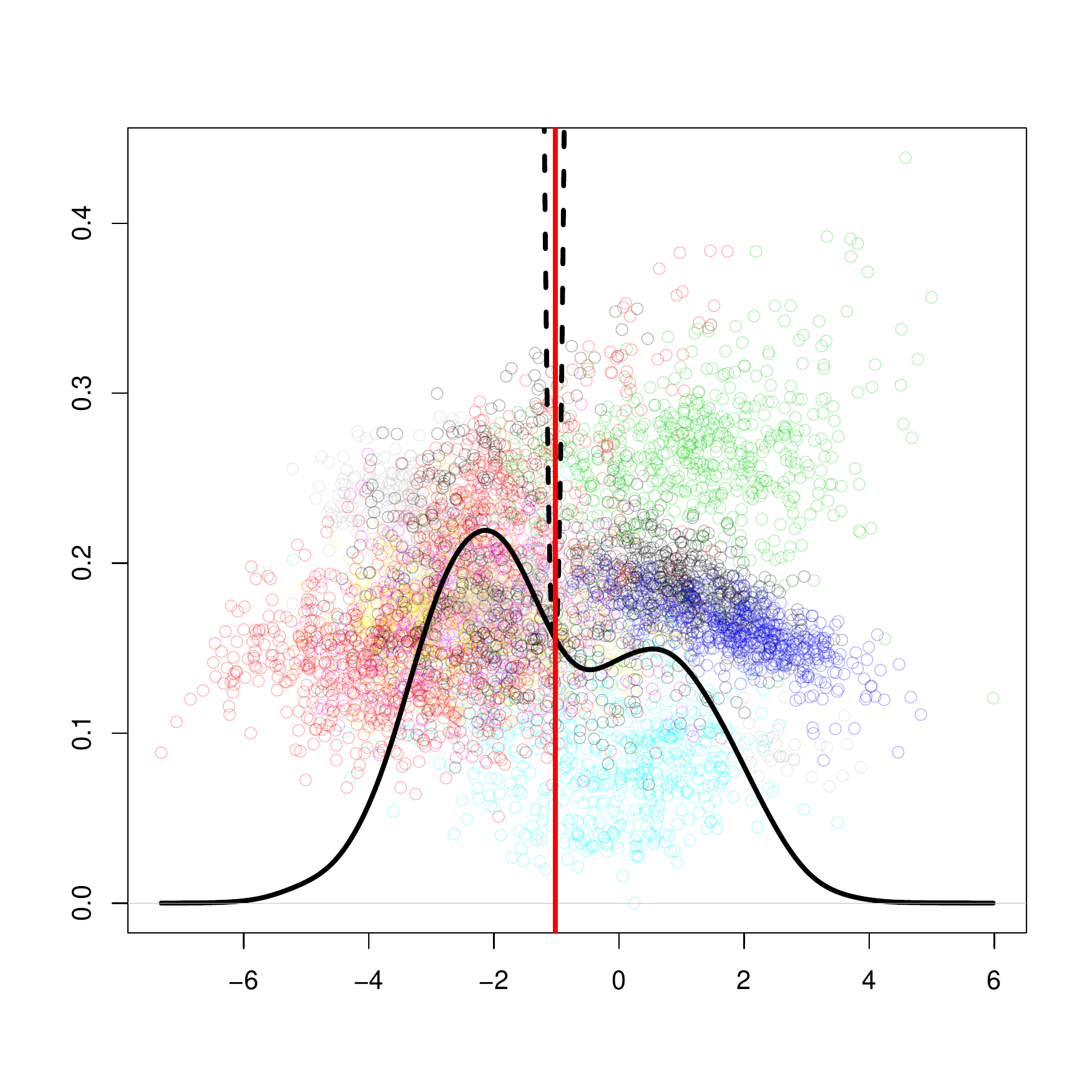}\label{fig:opt2}}
\\
\subfigure[Projection after 25 iterations with
$\alpha=10^{-2}$]{\includegraphics[width = 7cm, height = 6cm]{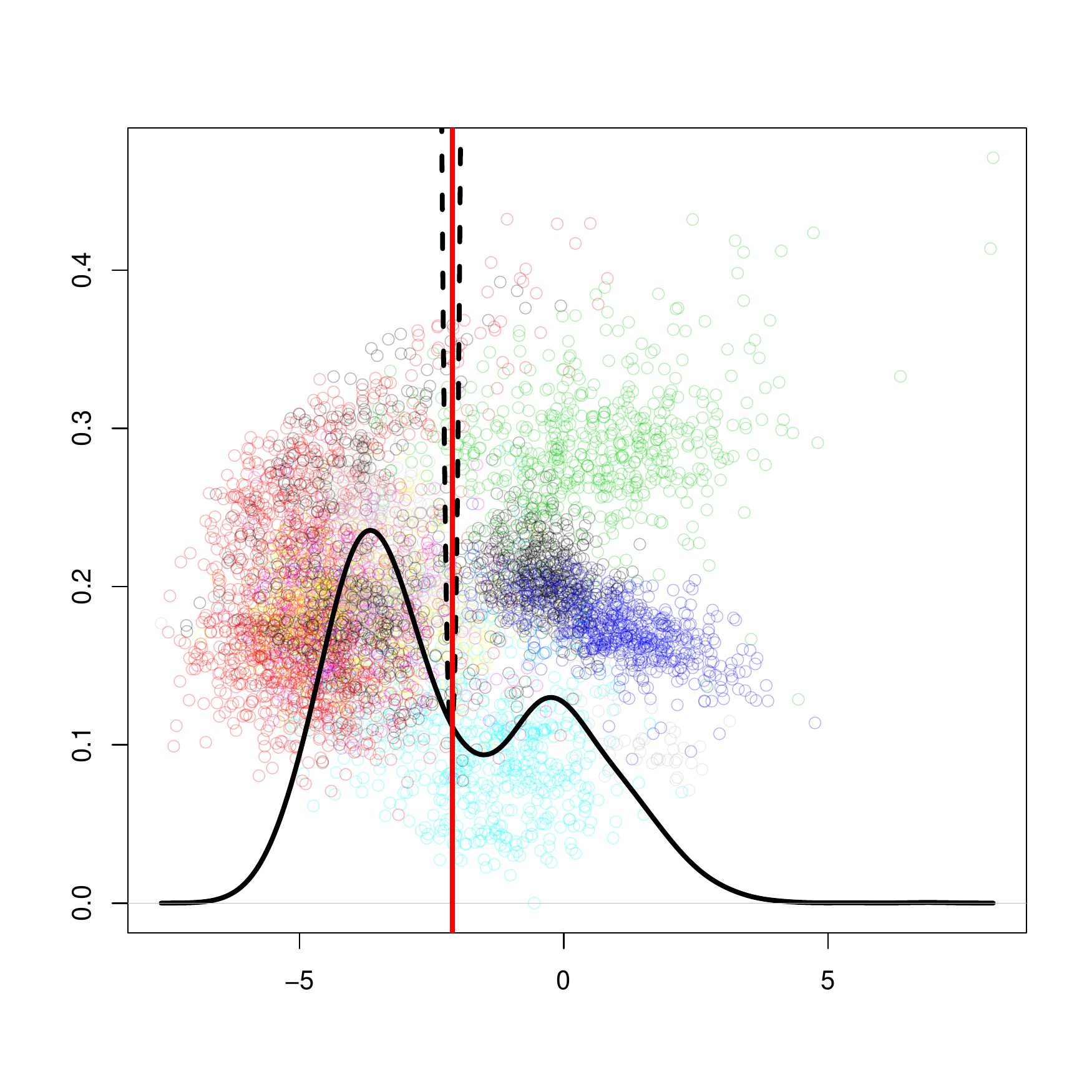}\label{fig:opt3}}
\subfigure[Final projection with $\alpha=\alpha_{\max}$]{\includegraphics[width = 7cm, height = 6cm]{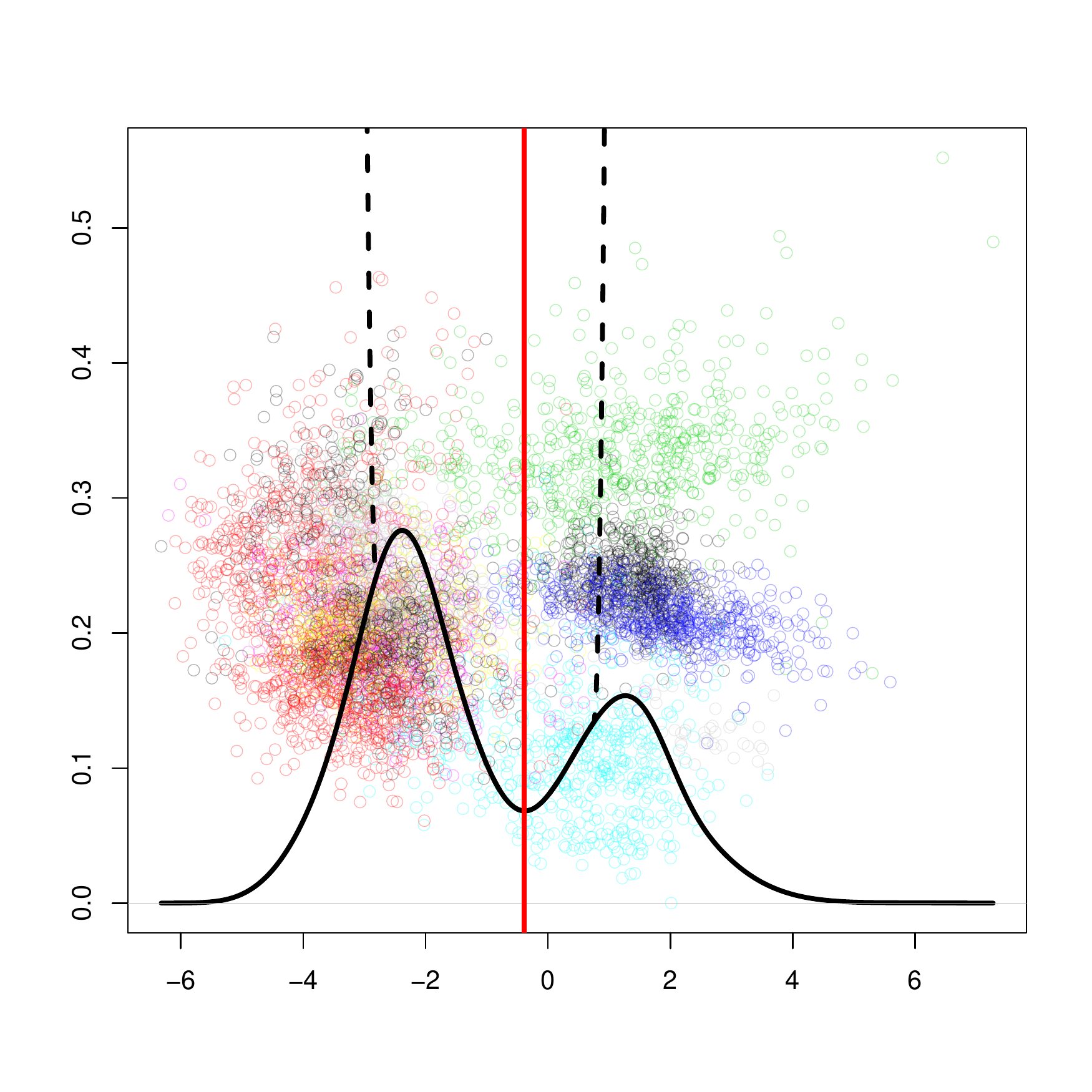}\label{fig:opt4}}
\caption{Evolution of the minimum density hyperplane through consecutive iterations.}\label{fig:opti}
\end{figure}

In all experiments we set the bandwidth parameter to~$h = 0.9
\hat{\sigma}_{\textrm{pc}_1}n^{-1/5}$, where
$\hat{\sigma}_{\textrm{pc}_1}$ is the estimated standard deviation of
the data projected onto the first principal component. This bandwidth
selection rule is recommended when the density being approximated is
assumed to be multimodal~\citep{Silverman1986}.
The parameter $\eta$ controls the distance between the minimisers of 
$\argmin_{b \in \R} \fUL(\v,b)$ and $\argmin_{b \in
F(\v)} \hat{I}(\v,b)$, while larger values of $\epsilon$ increase the smoothness
of the penalised function $\fUL$. 
Values of $\eta$ 
close to zero affect the numerical stability of the one-dimensional
optimisation problem, due to the term $\frac{L}{\eta^\epsilon}$ in $\fUL$ becoming very large.
We used $\eta=10^{-2}$ and $\epsilon=1-10^{-6}$ to avoid numerical instability.
Beyond these numerical problems
the values of $\eta$ and $\epsilon$ do not affect the solutions obtained through~MDP$^2$.


\subsubsection{Performance Evaluation}

We compare the performance of \MDP for clustering with the following methods:
\begin{enumerate}
\addtolength{\itemsep}{-0.6\baselineskip}
\item $k$-means++~\citep{ArthurV2007}, a version of $k$-means that
is guaranteed to be $\mathcal{O}(\log k)$-competitive to the optimal
$k$-means clustering.

\item The adaptive linear discriminant analysis guided $k$-means
	(LDA-$k$m)~\citep{DingL2007}. LDA-$k$m attempts to discover 
	the most discriminative linear subspace for
	clustering by iteratively using $k$-means, 
	to assign labels to observations, and LDA to
	identify the most discriminative subspace.

\item The principal direction divisive partitioning (PDDP)~\citep{Boley1998}, and
	the density-enhanced PDDP (dePDDP)~\citep{TasoulisTP2010}. Both methods
	project the data onto the first principal component. PDDP splits at the
	mean of the projections, while dePDDP splits at the lowest local minimum 
	of the one-dimensional density estimator. 

\item The iterative support vector regression algorithm for 
	MMC~\citep{ZhangTK2009} using the inner product and 
	Gaussian kernel, iSVR-L and iSVR-G respectively.
	Both are initialised with the output of 2-means++.
\item Normalised cut spectral clustering (SCn)~\citep{NgJW2002} using the
	Gaussian affinity function, and the automatic bandwidth selection
	method of~\cite{Zelnik2004}. This choice of kernel and bandwidth produced
	substantially better performance than alternative choices considered.
	For data sets that are too large for the eigen decomposition of the Gram
	matrix to be feasible we employed the Nystr{\"o}m
	method~\citep{FowlkesBSM2004}.

\end{enumerate}

We also considered the density-based clustering algorithm PdfCluster~\citep{MenardiA2014}, but this algorithm could not
be executed on the larger data sets and so its performance is not reported in this
paper.
With the exception of SCn and iSVR-G, the methods considered bi-partition the
data through a hyperplane in the original feature space. For the 2-means and
LDA-2m algorithm the hyperplane separator bisects the line segment joining the
two centroids.
iSVR-L directly seeks the maximum margin hyperplane in the original space,
while iSVR-G seeks the maximum margin hyperplane in the feature space defined
by the Gaussian kernel.
PDDP and dePDDP use a hyperplane whose normal vector is the first principal
component. PDDP uses a fixed split point while dePDDP uses the
hyperplane with minimum density along the fixed projection direction.

Table~\ref{tbl:cluster} reports the performance of the considered methods with
respect to the success ratio (SR) and the binary V-measure (V-m) on the fourteen
data sets. In addition Figures~\ref{fig:clMetr}
and~\ref{fig:clReg} provide summaries of the overall performance on all data sets
using boxplots of the raw performance measures as well as the associated \textit{regret}.
The regret of an algorithm on a given data set is defined as the difference
between the best performance attained on this data set and the performance of
this algorithm. By comparing against the best performing clustering algorithm regret
accommodates for differences in difficulty between clustering problems, while
also making use of the magnitude of performance differences between algorithms.
The distribution of performance with respect to both SR and V-m is negatively
skewed for most methods, and as a result the median is higher than the mean
(indicated with a red dot).

It is clear from Table~\ref{tbl:cluster} that no single method is consistently superior
to all others, although
\MDP achieves the highest or tied highest performance on seven data sets 
(more than any other method). More importantly \MDP is among the best performing
methods in almost all cases. This fact is better captured by
the regret distributions in Figure~\ref{fig:clReg}. Here we see that the
average, median, and maximum regret of \MDP is substantially lower than
any of the competing methods.
In addition \MDP achieves the highest mean and
median performance with respect to both SR and V-m, while
also having much lower variability in performance when compared with
most other methods.

Pairwise comparisons between \MDP and other methods reveal some
less obvious facts. SCn achieves higher performance than \MDP in
more examples (six) than any other competing method, however it is much
less consistent in its performance, obtaining very poor performance on
five of the data sets. The iSVR maximum margin clustering approach is
arguably the
closest competitor to MDP$^2$. 
iSVR-L and iSVR-G achieve
the second and third highest average performance with respect to V-m
and SR respectively.
The PDDP algorithm is the second best performing method on average with
respect to SR, but performs poorly with respect to V-m.
The density enhanced variant, dePDDP, performs on average much worse than
MDP$^2$. This approach is similarly motivated by obtaining hyperplanes
with low density integral, and its low average performance
indicates the usefulness of searching for high quality projections as
opposed to always using the first principal component.
Finally, neither of the $k$-means variants appears to be competitive with \MDP in
general.

\setlength{\tabcolsep}{.1cm}
\setlength\extrarowheight{.1cm}
\begin{table}
\small
\begin{center}
\scalebox{.85}{
\begin{tabular}{|l|ll|ll|ll|ll|ll|ll|ll|ll|}
\hline
\hline
	& \mcl{2}{\MDP} & \mcl{2}{iSVR-L} & \mcl{2}{iSVR-G} & \mcl{2}{SCn}   & \mcl{2}{LDA-2m} & \mcl{2}{2-means++} & \mcl{2}{PDDP}  & \mcl{2}{dePDDP} \\
\hline
Data set	& SR & V-m	& SR      & V-m      & SR      & V-m      & SR     & V-m   & SR    & V-m          & SR    & V-m      &  SR    & V-m   & SR & V-m  \\
\hline
banknote   & {\bf 0.79} & {\bf 0.55} & 0    & 0    & 0.35       & 0          & 0.46       & 0.10       & 0          & 0.01       & 0.37    & 0.01    & 0.40 & 0.03 & 0       & 0.03 \\
\hline
br. cancer & {\bf 0.91} & {\bf 0.79} & 0.73 & 0.56 & 0.73       & 0.56       & 0          & 0.13       & 0.87       & 0.71       & 0.87    & 0.72    & 0.91 & 0.78 & 0.90    & 0.77 \\
\hline
forest     & 0.78       & 0.67       & 0.90 & 0.72 & {\bf 0.91} & {\bf 0.74} & 0.56       & 0.41       & 0.76       & 0.63       & 0.72    & 0.58    & 0.64 & 0.36 & 0       & 0   \\
\hline
image seg. & 0.89       & 0.72       & 0.82 & 0.59 & 0.88       & 0.71       & 0.92       & 0.87       & 0.78       & 0.58       & 0.78    & 0.71    & 0.87 & 0.67 & {\bf 1} & {\bf 1} \\
\hline
ionosphere & 0.48       & 0.13       & 0.47 & 0.13 & 0.47       & 0.13       & {\bf 0.55} & {\bf 0.22} & 0.47       & 0.12       & 0.47    & 0.12    & 0.47 & 0.12 & 0.42    & 0.09 \\
\hline
optidigits & {\bf 0.93} & {\bf 0.85} & 0.63 & 0.29 & 0.82       & 0.60       & 0          & 0          & 0.81       & 0.62       & 0.92    & 0.82    & 0.68 & 0.30 & 0       & 0   \\
\hline
pendigits  & 0.74       & 0.39       & 0.79 & 0.55 & {\bf 0.88} & {\bf 0.68} & 0.80       & 0.68       & 0.79       & 0.55       & 0.78    & 0.57    & 0.79 & 0.54 & 0.61    & 0.42 \\
\hline
satellite  & 0.89       & 0.75       & 0.73 & 0.40 & 0.73       & 0.40       & {\bf 0.92} & {\bf 0.86} & 0.73       & 0.40       & 0.87    & 0.81    & 0.71 & 0.37 & 0       & 0   \\
\hline
seeds      & 0.88       & 0.73       & 0.71 & 0.53 & 0.71       & 0.53       & 0.89       & 0.76       & {\bf 0.96} & {\bf 0.90} & 0.86    & 0.70    & 0.75 & 0.59 & 0.73    & 0.60 \\
\hline
smartphone & {\bf 0.99} & {\bf 0.97} & 0.99 & 0.95 & 0.99       & 0.96       & 0.99       & 0.94       & 0.99       & 0.97       & 0.99    & 0.94    & 0.99 & 0.95 & 0       & 0   \\
\hline
synth      & 0.98       & 0.94       & 0.94 & 0.83 & 0.94       & 0.83       & {\bf 1}    & {\bf 1}    & 0.88       & 0.76       & {\bf 1} & {\bf 1} & 0.69 & 0.51 & {\bf 1} & {\bf 1} \\
\hline
voting     & {\bf 0.70} & {\bf 0.43} & 0.46 & 0.09 & 0          & 0          & 0          & 0.05       & 0.69       & 0.41       & 0       & 0       & 0.70 & 0.40 & 0.68    & 0.38 \\
\hline
wine       & {\bf 0.77} & {\bf 0.61} & 0.70 & 0.52 & 0.69       & 0.50       & 0.67       & 0.48       & 0.66       & 0.48       & 0.68    & 0.49    & 0.65 & 0.46 & 0.68    & 0.49 \\
\hline
yeast      & {\bf 0.92} & {\bf 0.76} & 0.89 & 0.68 & 0.91       & 0.72       & 0.84       & 0.61       & 0.86       & 0.63       & 0.91    & 0.73    & 0.87 & 0.65 & 0       & 0   \\
\hline
\hline
\multicolumn{3}{c|}{Average Improvement} & 0.13 & 0.18 & 0.12 & 0.14 & 0.22 & 0.16 & 0.10 & 0.11 & 0.10 & 0.08 & 0.11 & 0.18 & 0.40 & 0.32 \\
\hline
\hline
\end{tabular}
}
\end{center}
\caption{Performance on the task of binary partitioning. (Ties in best performance were resolved by considering more decimal places)}\label{tbl:cluster}
\end{table}

\begin{figure}[!h]
\centering
\subfigure[Raw Performance Measure]{\includegraphics[scale=0.28]{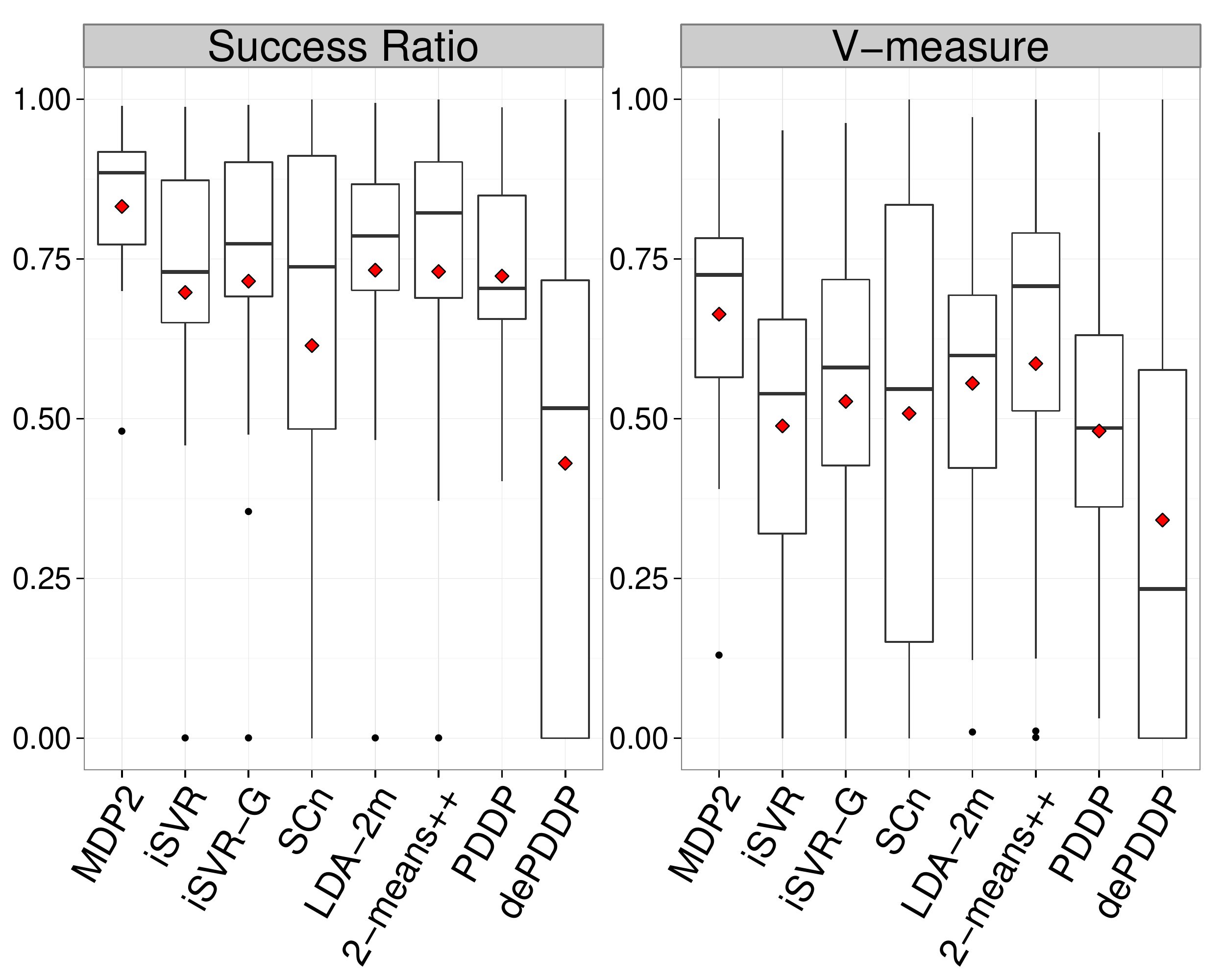}\label{fig:clMetr}} \hfill
\subfigure[Regret]{
\includegraphics[scale=0.28]{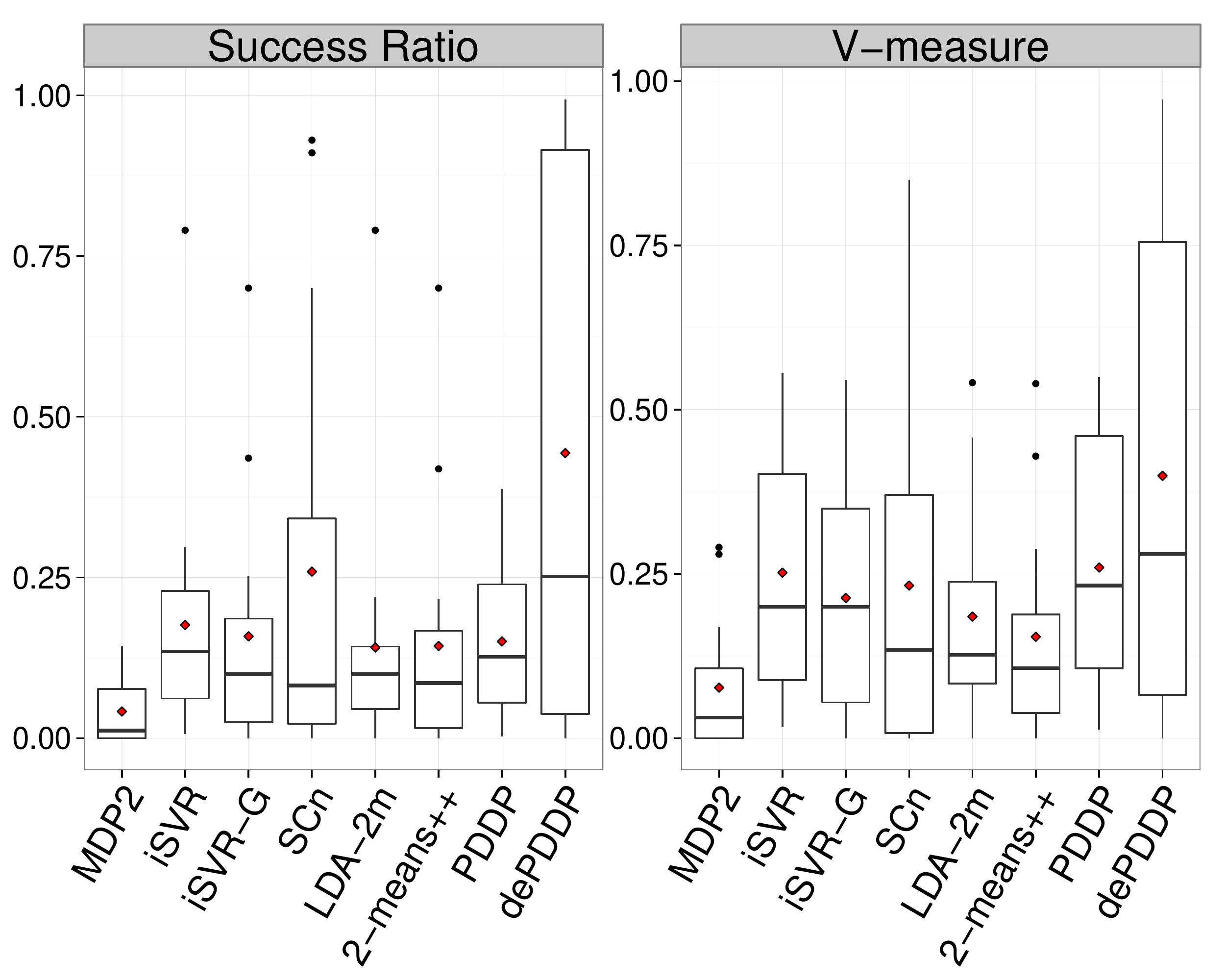}\label{fig:clReg}
}
\caption{Performance and Regret Distributions for all Methods Considered}
\label{fig:cluster}
\end{figure}

\subsection{Semi-Supervised Classification}\label{ssec:expSSL}

In this section we evaluate MDHs for semi-supervised classification. 
We compare MDHs against three state-of-the-art semi-supervised classification
methods: Laplacian Regularised Support Vector Machines
(LapSVM)~\citep{BelkinNS2006}, Simple Semi-Supervised Learning
(SSSL)~\citep{JiYLJH2012}, and Correlated Nystr{\" o}m Views
(XNV)~\citep{mcwilliams2013}. For all methods the inner product kernel was used
to render the resulting classifiers linear, and thereby comparable to our
method.
As the MDH is asymptotically equivalent to a linear \S3VM we also considered
the continuous formulation for the estimation of a \S3VM proposed
by~\cite{ChapelleZ2005}. These results are omitted as this method was not
competitive on any of the considered data sets.

\subsubsection{Parameter Settings for \MDP}

The existence of a few labelled examples enables an informed initialisation of
MDP$^2$. We consider the first and second principal components as well as the
weight vector of a linear SVM trained on the labelled examples only, and initialise
\MDP with the vector that minimises the value of the projection index,$~\phiSSL$.
The penalty parameter $\gamma$ is first set to~$0.1$ and with this setting $\alpha$
is progressively increased in the same way as for clustering. After this, $\alpha$
is kept at $\alpha_{\max}$ and $\gamma$ is increased to~1 and then~10.
Thus the emphasis is initially on finding
a low density hyperplane with respect to the marginal density $\hat{p}(\x)$.
As the algorithm progresses the emphasis on correctly classifying the labelled
examples increases, so as to obtain a hyperplane with low training error within
the region of low density already determined.

\subsubsection{Performance Evaluation}

To assess the effect on performance of the number of labelled examples, $\ell$,
we consider a range of values.
We compare the methods using the subset of data sets used in the previous section in
which the size of the smallest class exceeds 100. In total eight data sets are
used. 
For each value of $\ell$, 30 random partitions into labelled and unlabelled data
are considered.
As classes are balanced in the data sets considered, performance is measured
only in terms of classification error on the unlabelled data. 
For data sets with more than two classes all pairwise combinations of classes
are considered and aggregate performance is reported.

\begin{figure}
\vspace{-30pt}
\centering
\subfigure[voting]{\includegraphics[width = 7cm, height = 5cm]{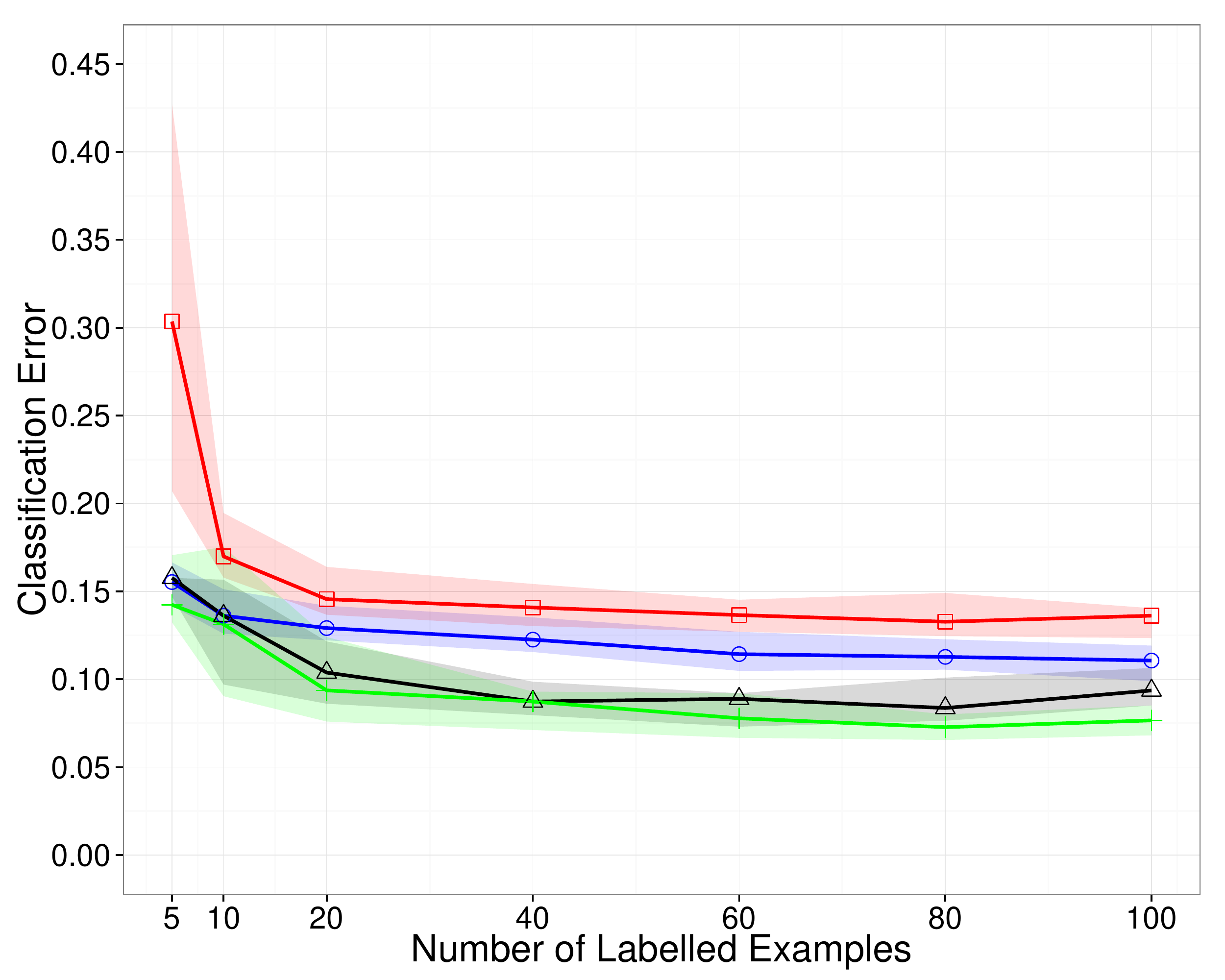}}
\subfigure[banknote]{\includegraphics[width = 7cm, height = 5cm]{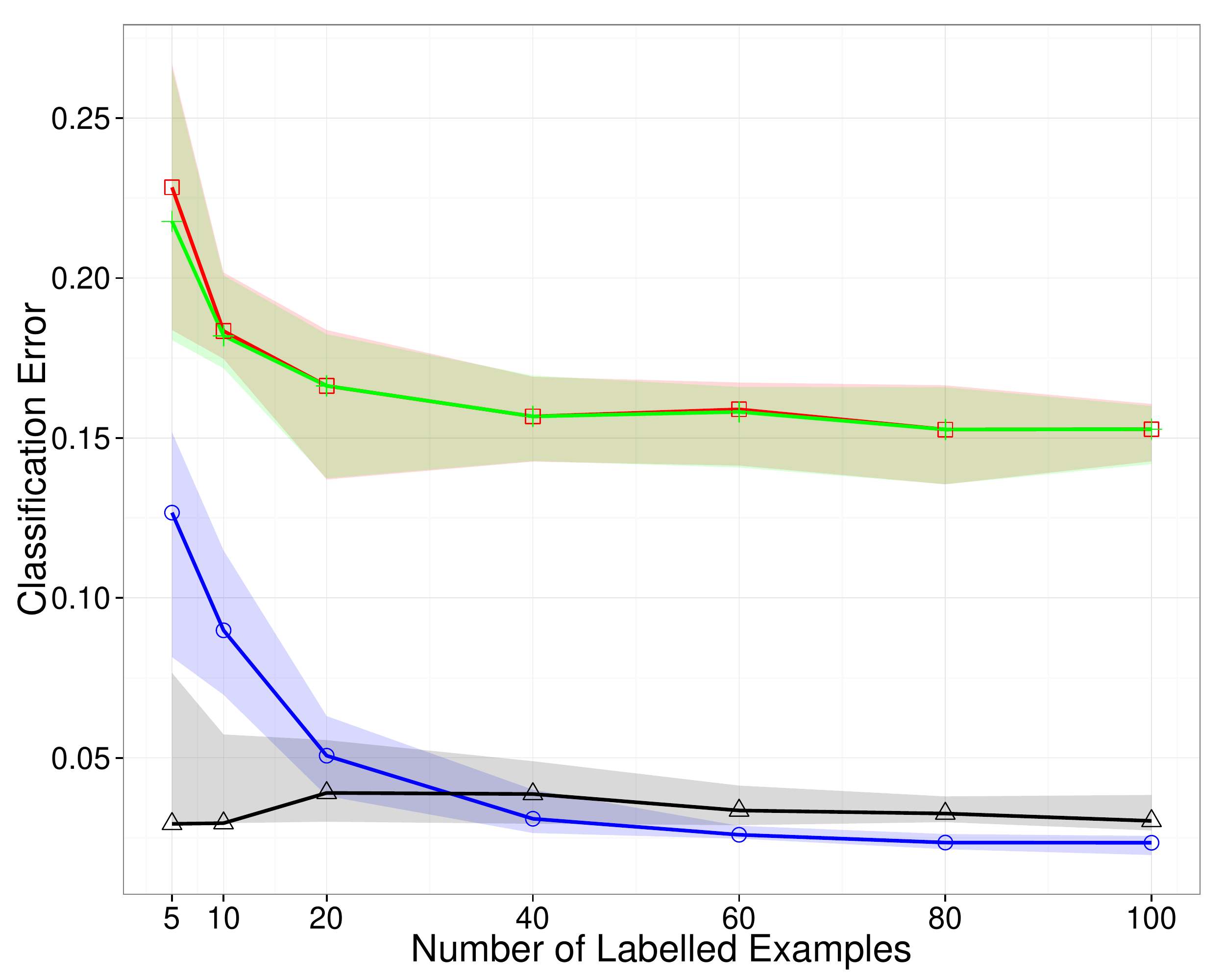}}
\subfigure[breast cancer]{\includegraphics[width = 7cm, height = 5cm]{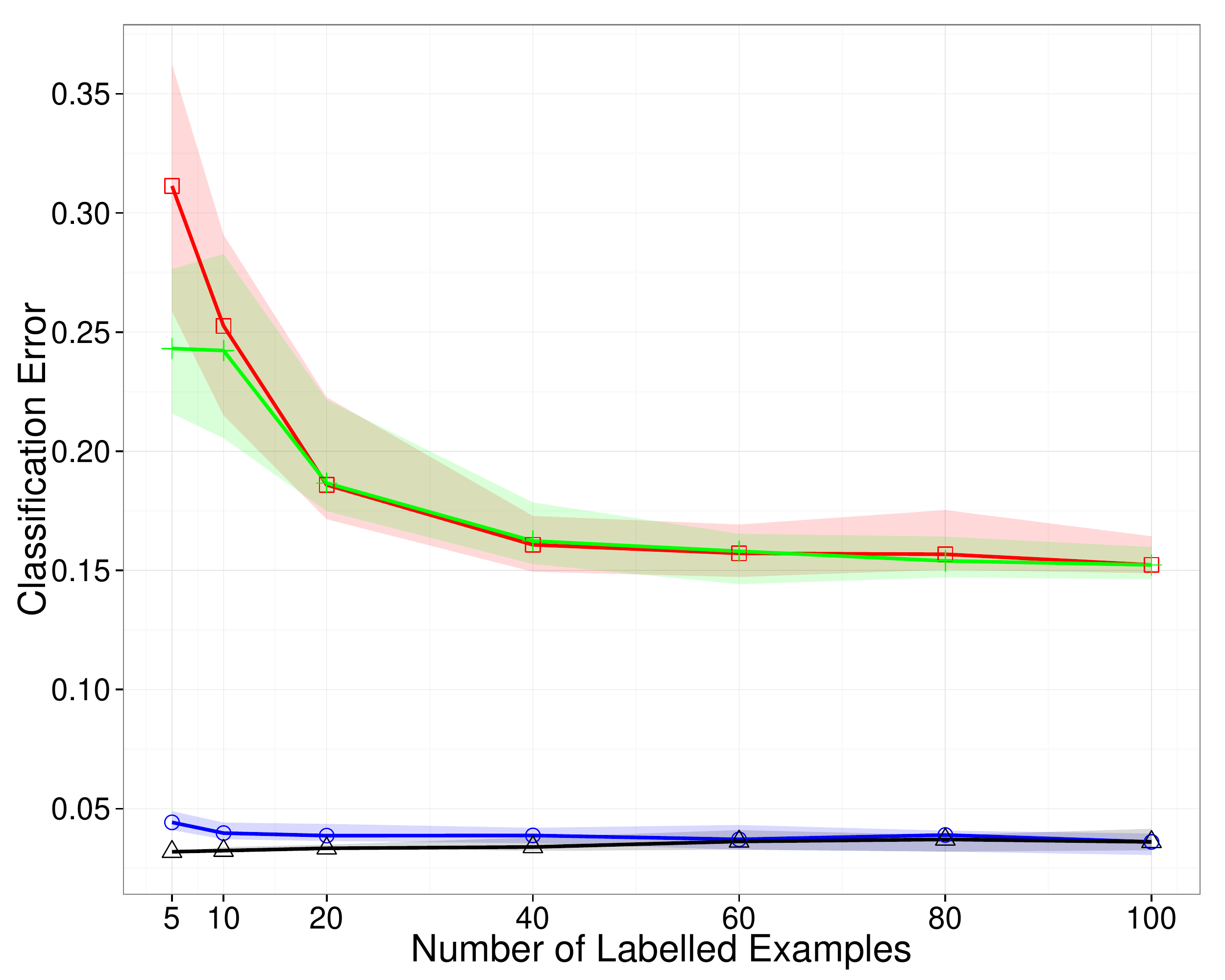}}
\subfigure[ionosphere]{\includegraphics[width = 7cm, height = 5cm]{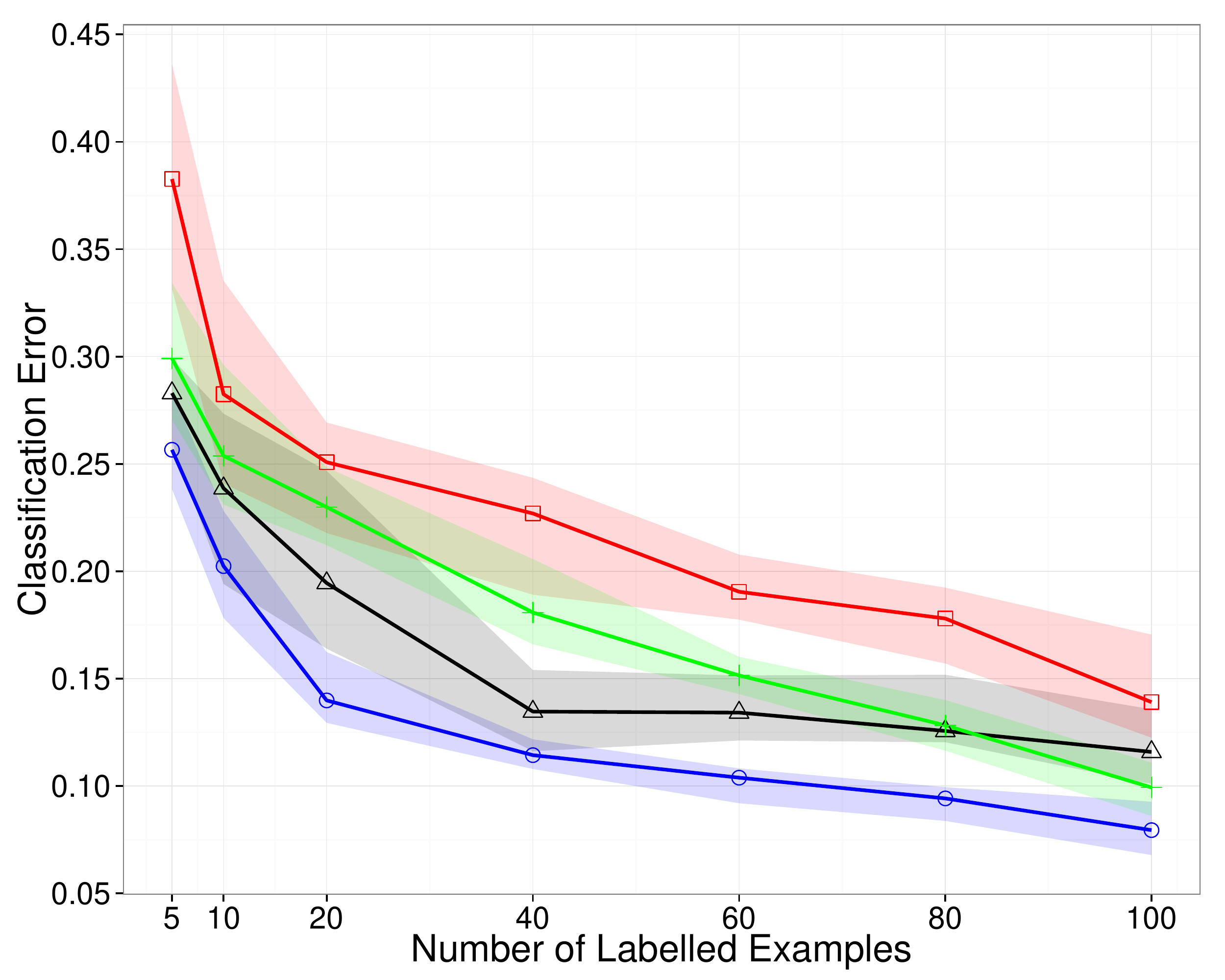}}
\MDP median (---{\tiny $\triangle$}---), LapSVM median (\blue{---$\circ$---}), SSSL median (\red{---{\tiny $\square$}---}),
XNV median~({\color{green}---{\small $+$}---}), with corresponding interquartile ranges given by shaded regions.
\caption{Classification error for different number of labelled examples for data sets with two clusters.}\label{ssl:regret1}
\vspace{-25pt}
\end{figure}

Figure~\ref{ssl:regret1} provides plots of the median and interquartile range
of the classification error for values of $\ell$ between 5 and 100 for the four
data sets with two classes.
Overall \MDP appears to be most competitive when the number of labelled
examples is small. 
In addition, \MDP is comparable with the best performing method in almost every
case. The only exception is the ionosphere data set where LapSVM outperforms
\MDP for all values of $\ell$.  Figure~\ref{ssl:regret2} provides plots of the
median and interquartile range of the aggregate classification error on
data sets containing more than two classes. 
As these data sets are larger we consider up to 300 labelled examples. 
Note that the interquartile range for XNV is not depicted for the
satellite data set.  The variability of performance of XNV on this data set was
so high that including the interquartile range would obscure all other
information in the figure.
\MDP exhibits the best performance overall, and obtains the lowest median
classification error, or tied lowest, for all data sets and values of $\ell$.

\begin{figure}
\vspace{-30pt}
\centering
\subfigure[pendigits]{\includegraphics[width = 7cm, height = 5cm]{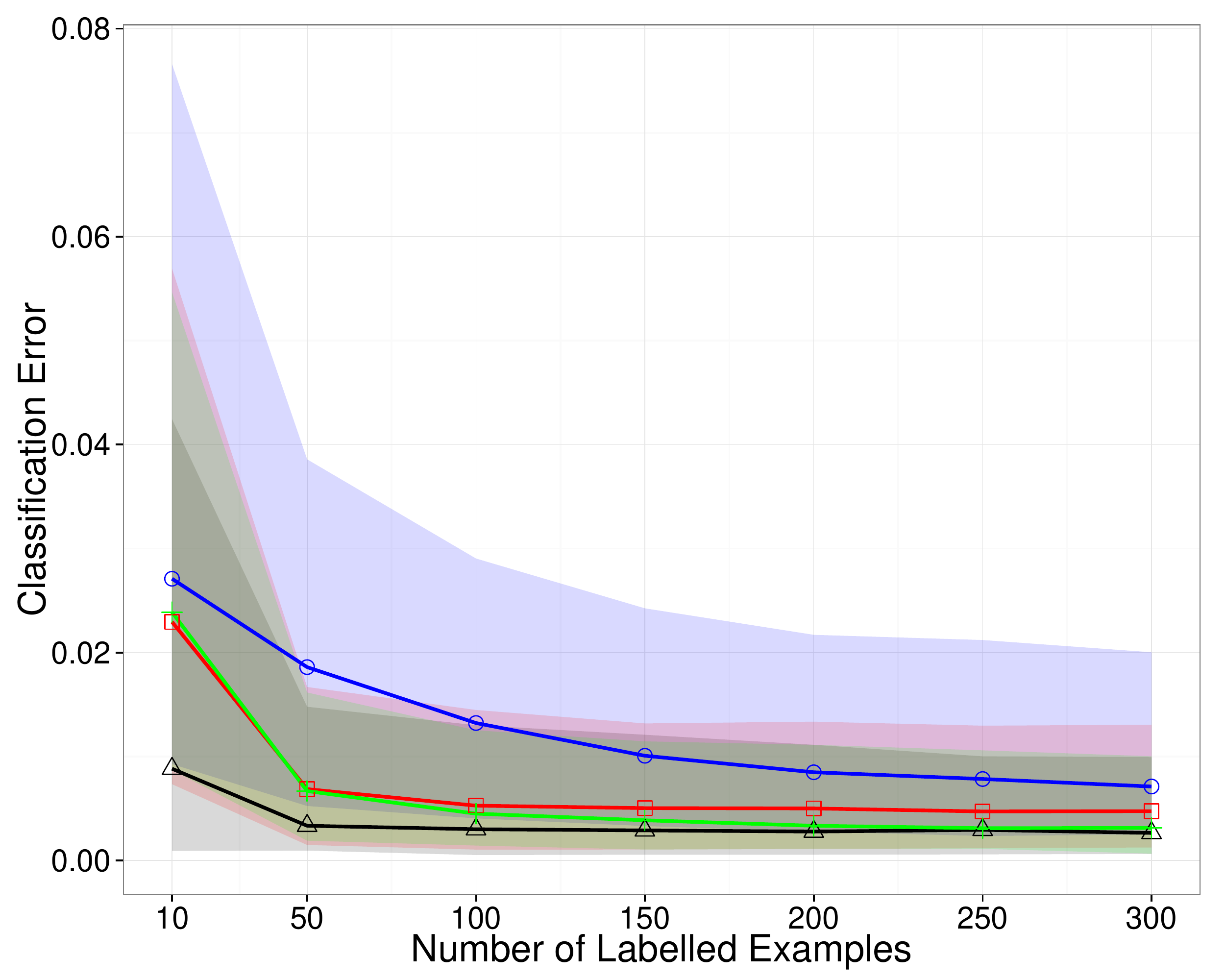}}
\subfigure[optidigits]{\includegraphics[width = 7cm, height = 5cm]{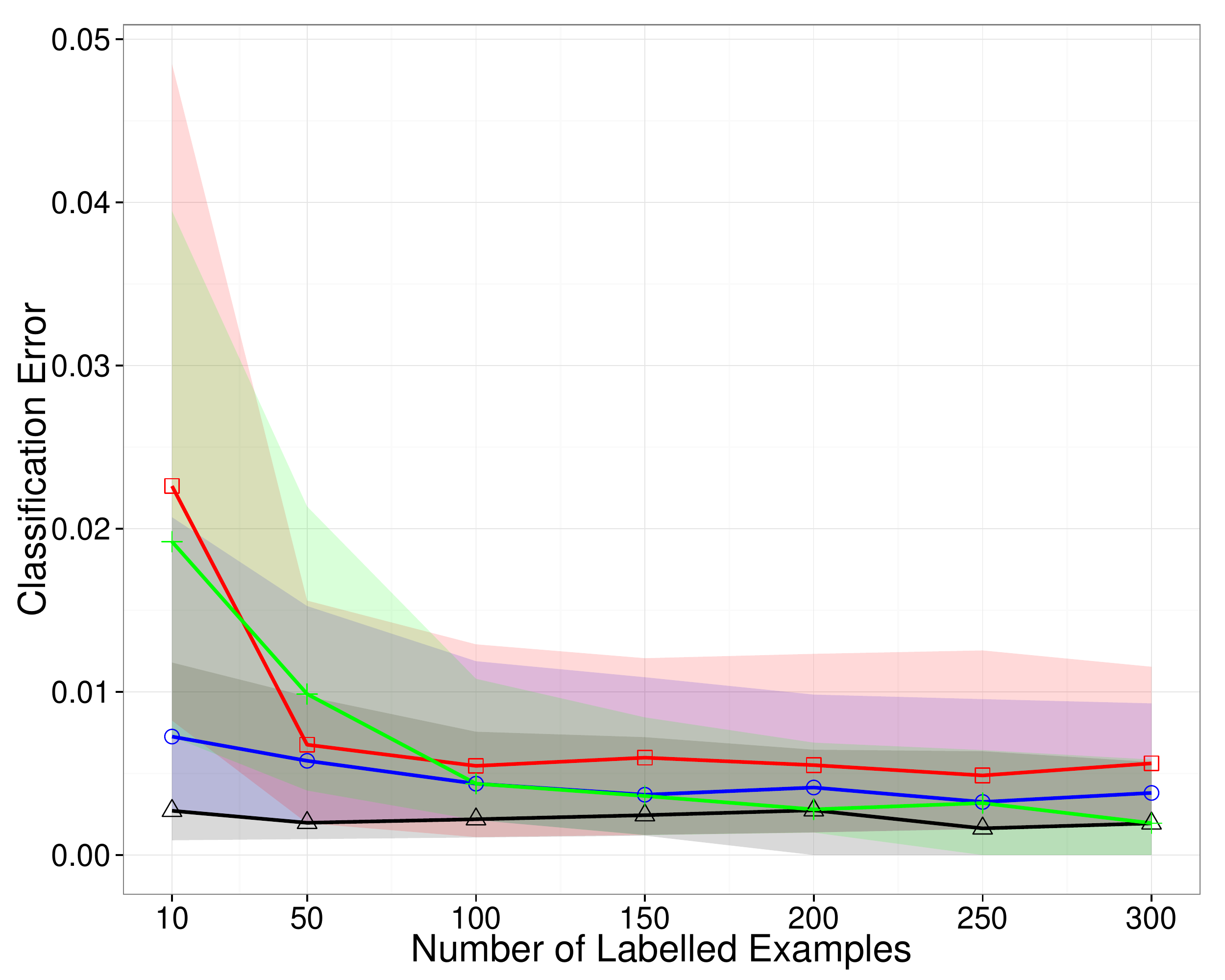}}
\subfigure[satellite]{\includegraphics[width = 7cm, height = 5cm]{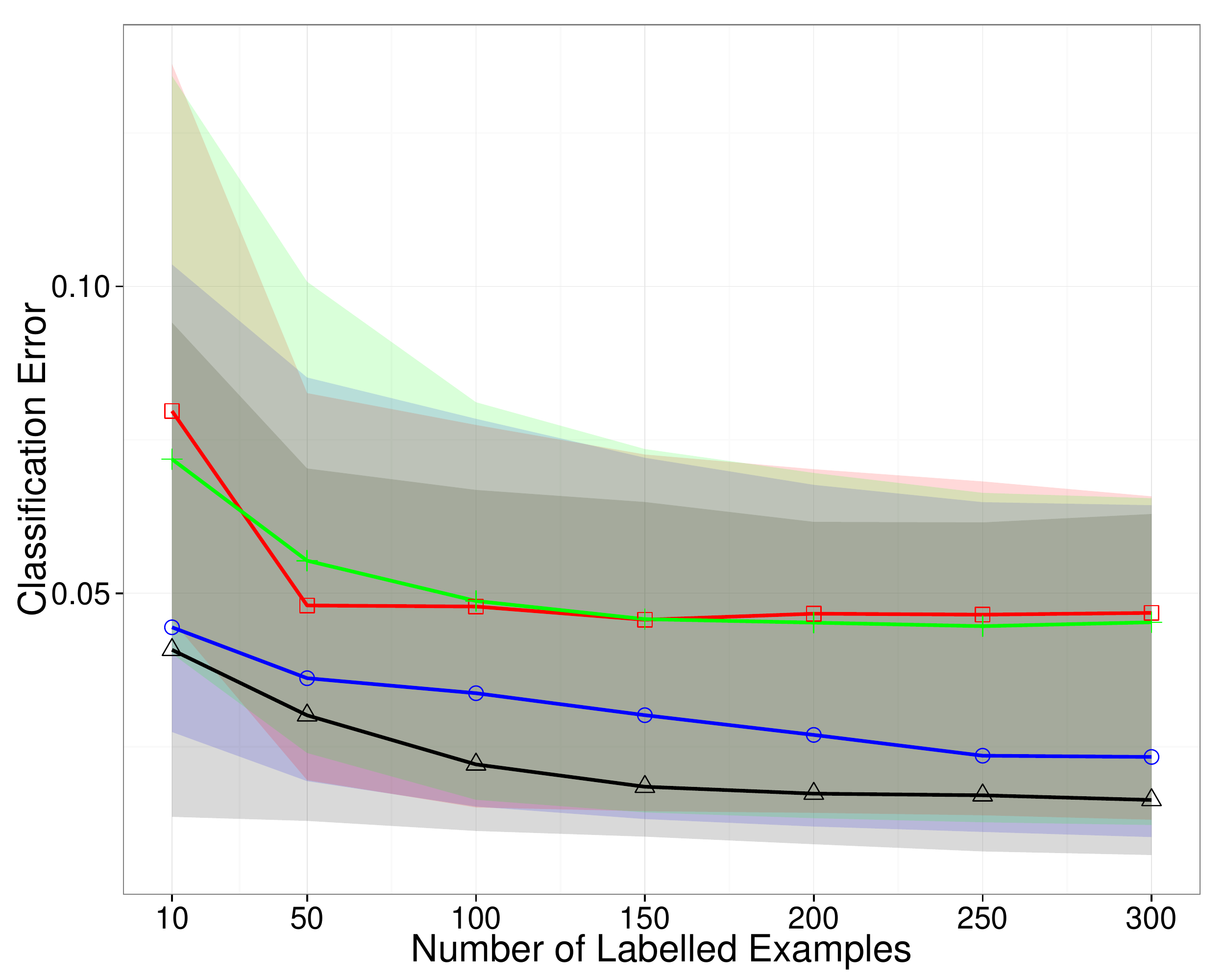}}
\subfigure[image segmentation]{\includegraphics[width = 7cm, height = 5cm]{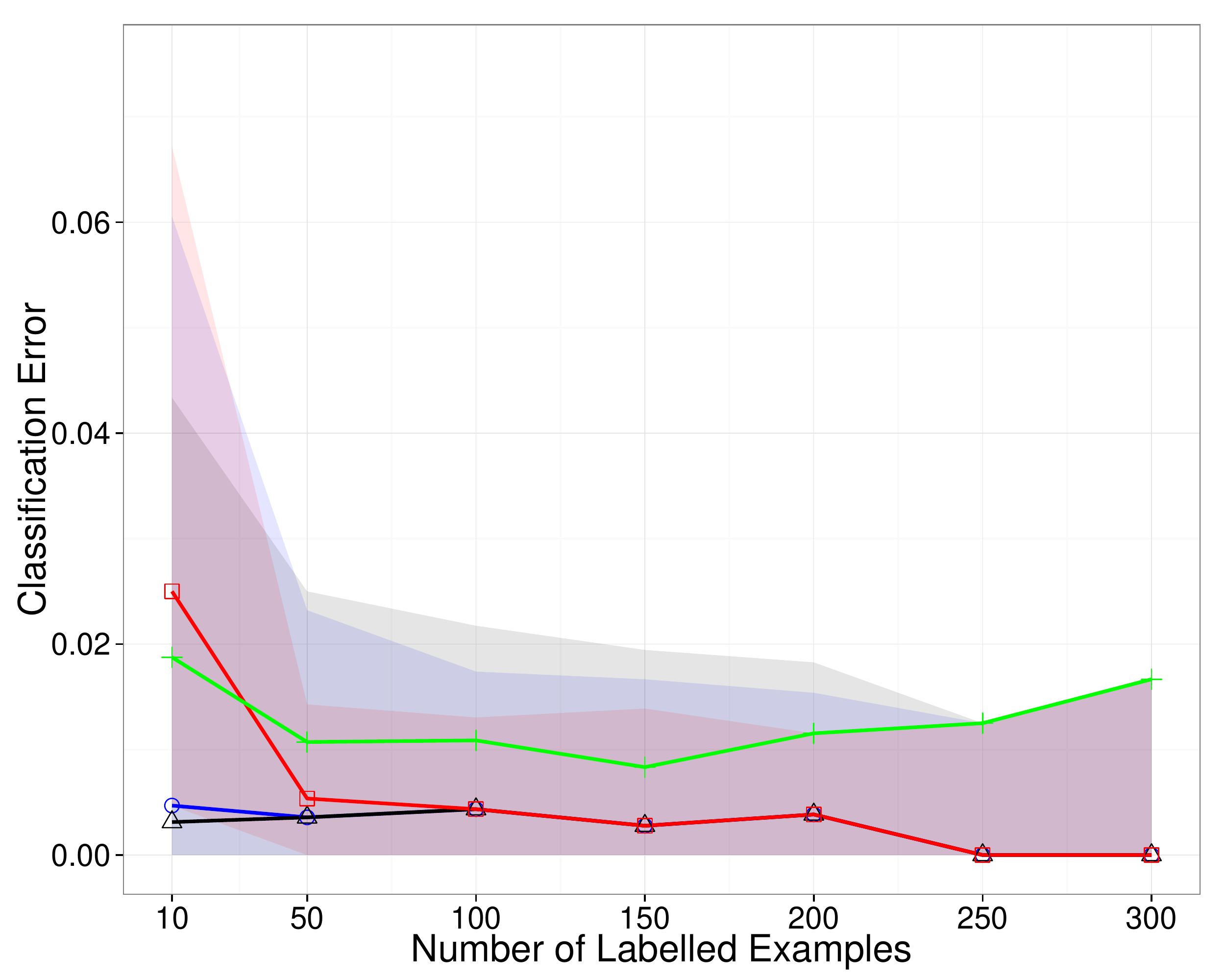}}
\MDP median (---{\tiny $\triangle$}---), LapSVM median (\blue{---$\circ$---}), SSSL median (\red{---{\tiny $\square$}---}), XNV median ({\color{green}---{\small $+$}---}), with corresponding interquartile ranges given by shaded regions.
\caption{Classification error for different numbers of labelled examples over all pairwise combinations of classes.}\label{ssl:regret2}
\vspace{-30pt}
\end{figure}


\subsection{Summary of Experimental Results}

We evaluated the performance of the \MDP formulation for finding
MDHs for both clustering and semi-supervised classification, on
a large collection of benchmark data sets, and in comparison with
state-of-the-art methods for both problems. 

For clustering, we found that no single method was consistently superior to all
others. This is a result of the vastly differing nature of the data sets in
terms of size, dimensionality, number and shape of clusters, etc. \MDP
achieved the best performance on more data sets than any of the competing
methods, and importantly was competitive with the best performing method in
almost every data set considered. All other methods performed poorly in at least
as many examples. Boxplots of both the raw performance and performance regret,
which measures the difference between each method and the best performing
method on each data set, allowed us to summarise the comparative performance of
the different methods across data sets. The mean and median raw performance of
\MDP is substantially higher than the next best performing method, and the
regret is also substantially lower.

In the case of semi-supervised classification it was apparent that \MDP is
extremely competitive when the number of labelled examples is (very) small, but
that in some cases its performance does not improve as much as that of the
other methods considered, when the labelled examples become more abundant.
Our experiments suggest that overall \MDP is very competitive with the
state-of-the-art for semi-supervised classification problems.

\section{Conclusions}\label{sec:concl}

We proposed a new hyperplane classifier for clustering and semi-supervised
classification. The proposed approach is motivated by determining low density
linear separators of the high-density clusters within a data set.
This is achieved by minimising the integral of the empirical density along the
hyperplane,
which is computed through kernel density estimation. 
To the best of our knowledge this is the first direct implementation of the low
density separation assumption that underlies high-density clustering and
numerous influential semi-supervised classification methods.
We show that the minimum density hyperplane is asymptotically
connected with maximum margin hyperplane, thereby establishing
an important link between the proposed approach, maximum margin clustering, and
semi-supervised support vector machines. 

The proposed formulation allows us to evaluate the integral of the density on a
hyperplane by projecting the data onto the vector normal to the hyperplane, and
estimating a univariate kernel density estimator.
This enables us to apply our method effectively and efficiently on data sets of
much higher dimensionality than is generally possible for density based
clustering methods.
To mitigate the problem of convergence to locally optimal solutions we proposed
a projection pursuit formulation.

We evaluated the minimum density hyperplane approach on a large collection of
benchmark data sets. The experimental results obtained indicate that the method
is competitive with state-of-the-art methods for clustering and semi-supervised
classification.
Importantly the performance of the proposed approach displays low variability
across a variety of data sets, and is robust to differences in data size,
dimensionality, and number of clusters. In the context of semi-supervised
classification, the proposed approach shows especially good performance when
the number of labelled data is small.

\acks{
We would like to thank the reviewers for their insightful comments which
substantially improved this paper.
We also thank Prof. David Leslie, and Dr. Teemu Roos for valuable comments and
suggestions on this work.
Nicos Pavlidis would like to thank the Isaac Newton Institute for Mathematical
Sciences, Cambridge, for support and hospitality during the programme `Inference
for Change-Point and Related Processes', where part of the work on this paper was undertaken.
David Hofmeyr gratefully acknowledges the support of the EPSRC funded
EP/H023151/1 STOR-i centre for doctoral training, as well as the Oppenheimer
Memorial Trust.
The underlying code and data are openly available from Lancaster
University data repository at \url{http://dx.doi.org/10.17635/lancaster/researchdata/97}.
}

\appendix


\section{Proof of Theorem~\ref{thm:convergence}}\label{app:thm:convergence}

Before proving Theorem~\ref{thm:convergence} we require the following two
technical lemmata which establish some algebraic properties of the maximum
margin hyperplane.  The following lemma shows that any hyperplane orthogonal to
the maximum margin hyperplane results in a different partition of the support
points of the maximum margin hyperplane. 
The proof relies on the fact that if this statement does not hold then a
hyperplane with larger margin exists which is a contradiction.
Figure~\ref{fig:MMHorthog} provides an illustration of why this result holds.
(a)~Any hyperplane orthogonal to MMH generates a different partition of the
support points of MMH, e.g., the point highlighted in red in (b)~is grouped
with the lower three by the dotted line but with the upper two by the solid
line, the MMH. If an orthogonal hyperplane \textit{can} generate the same
partition~(c), then a larger margin hyperplane than the proposed MMH exists~(d).

\begin{figure}
\centering
\subfigure[]{\includegraphics[width = 3cm]{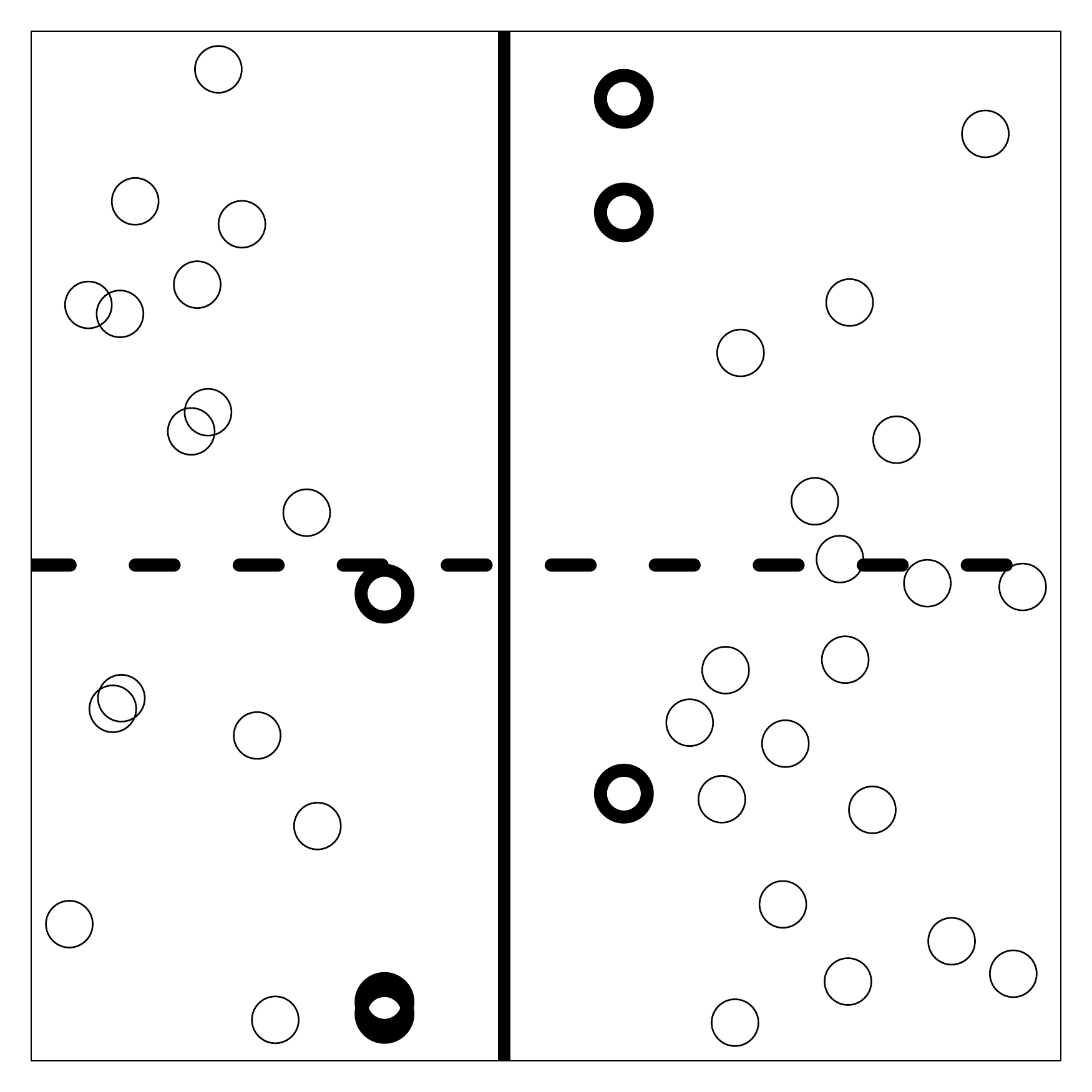}}
\subfigure[]{\includegraphics[width = 3cm]{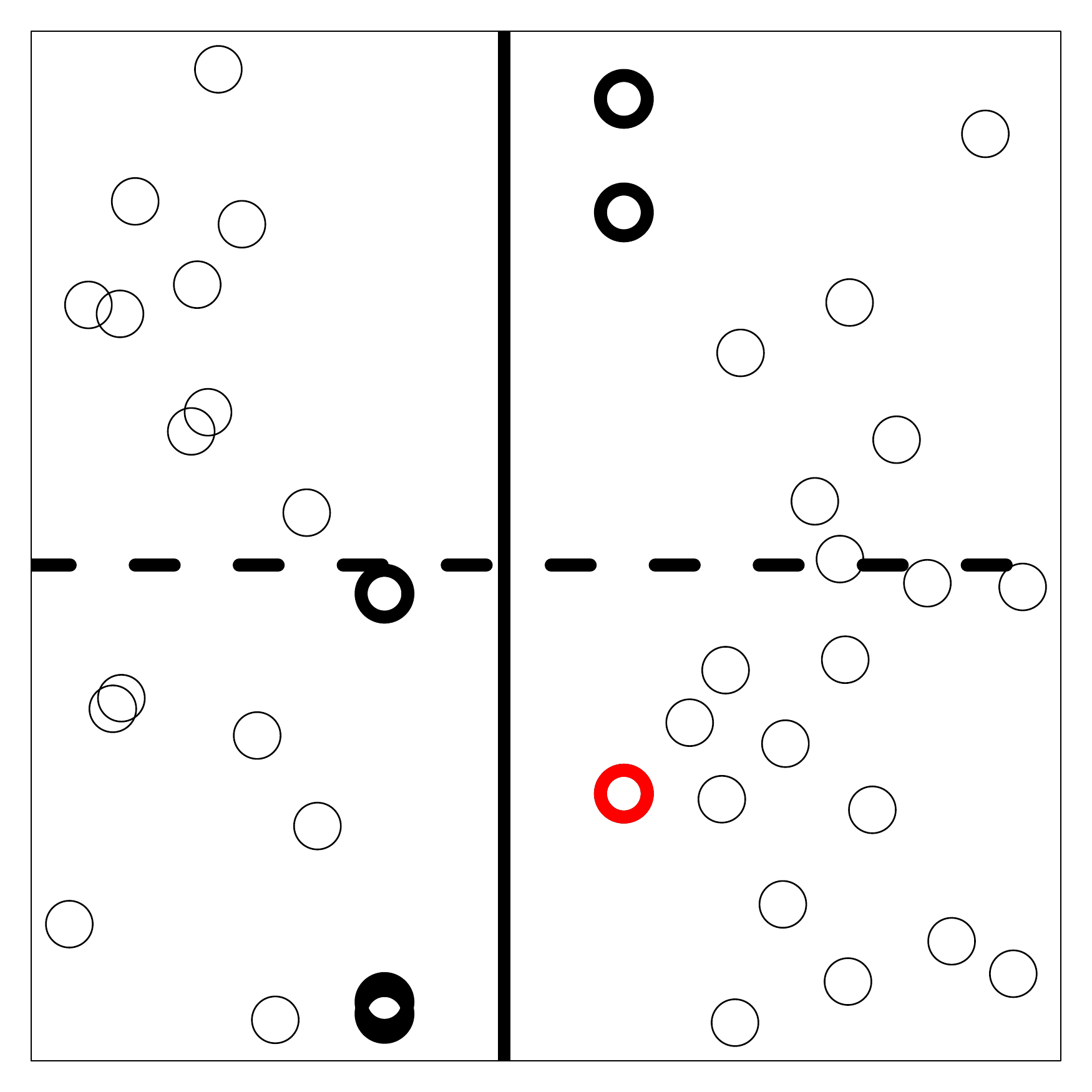}}\\
\subfigure[]{\includegraphics[width = 3cm]{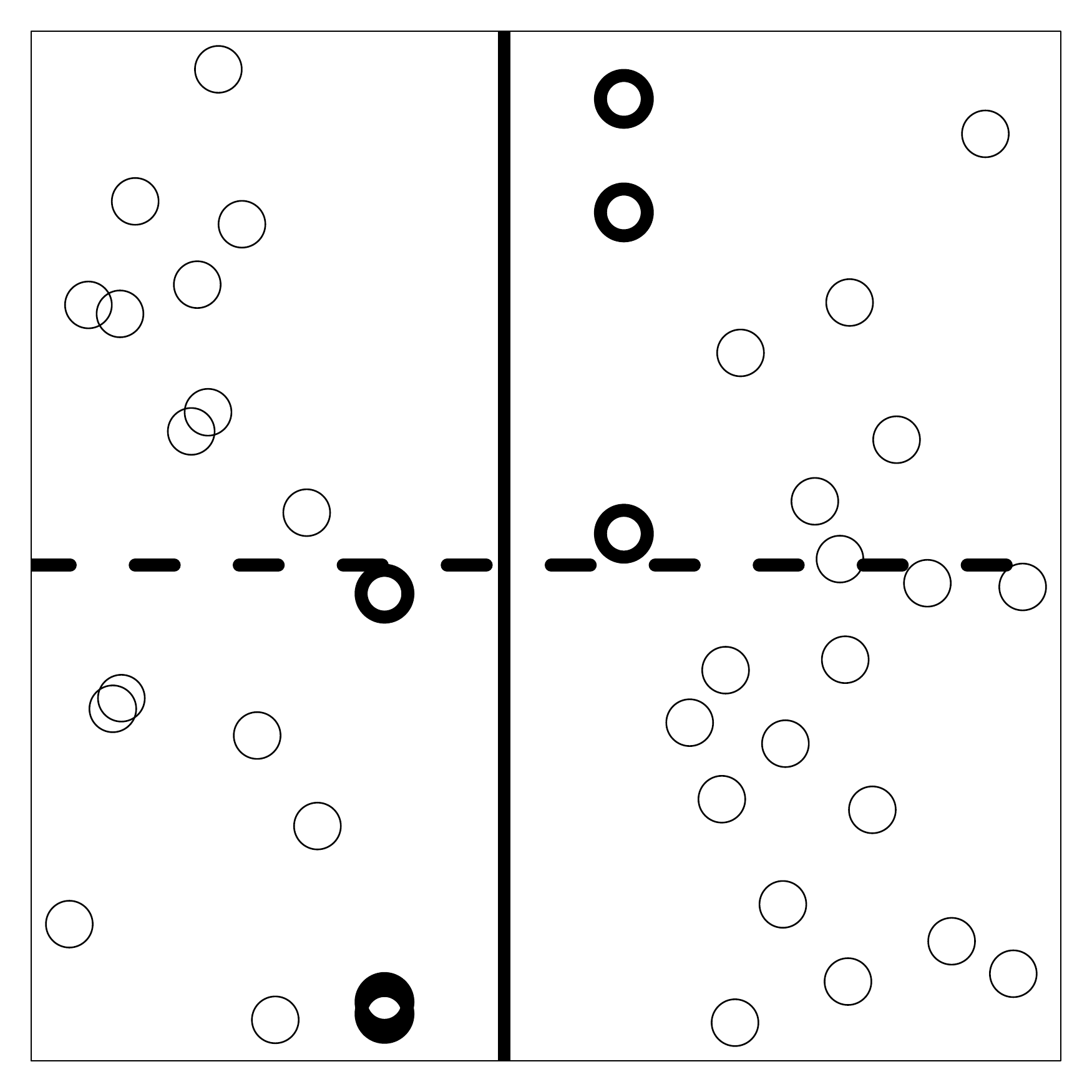}}
\subfigure[]{\includegraphics[width = 3cm]{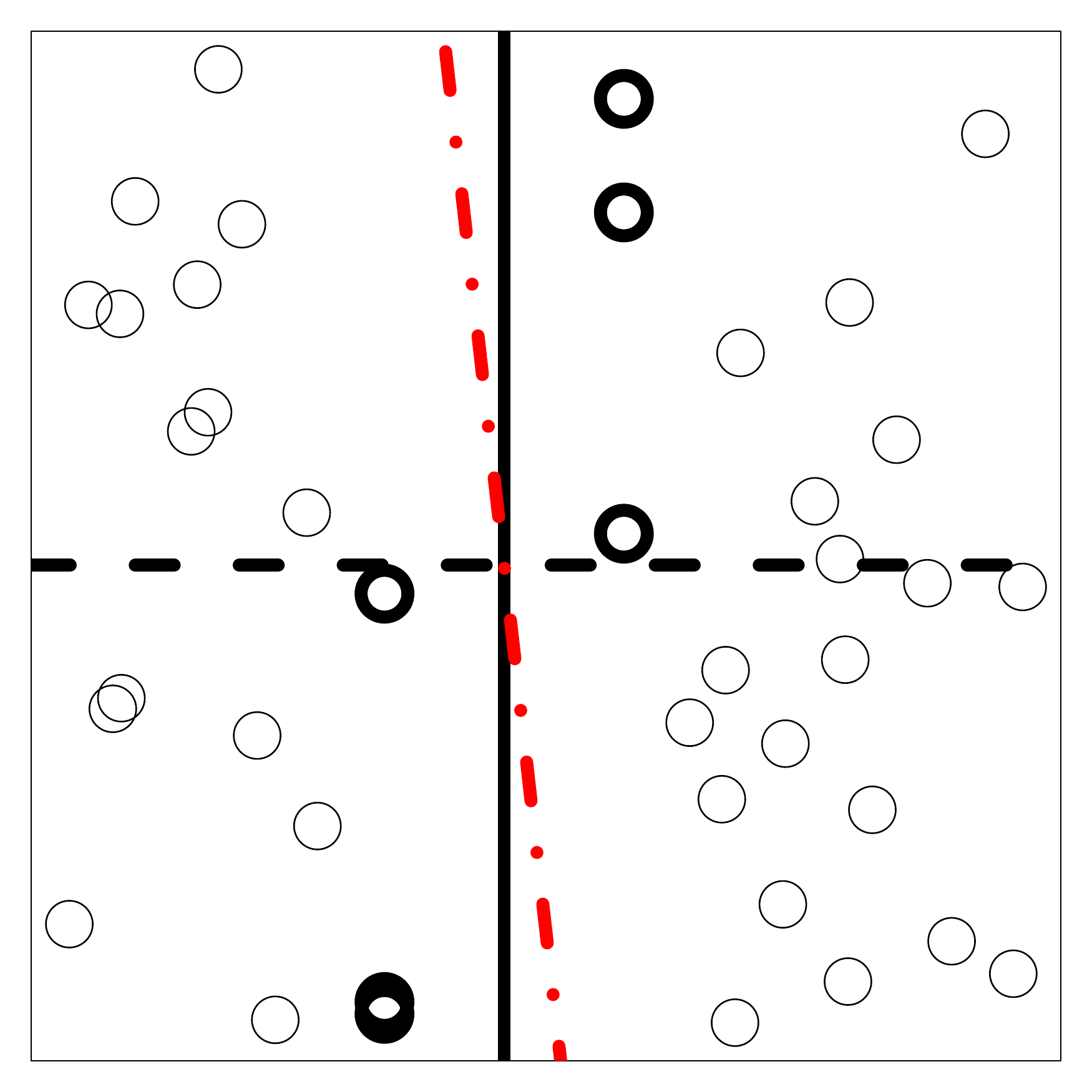}}\\
Proposed
MMH {\bf -----}, Orthogonal hyperplane {\bf - - -}, Hyperplane with larger \\ margin \red{\bf - $\cdot$ - $\cdot$ -}, Regular points $\bigcirc$, Support points $\pmb{\bigcirc}$,\\ Differently assigned support point \red{$\pmb{\bigcirc}$}
\caption{Two dimensional illustration of Lemma~\ref{lem:MMHorthog}}
\label{fig:MMHorthog}
\end{figure}

\begin{lemma}\label{lem:MMHorthog}

Suppose there is a unique hyperplane in $F$ with maximum margin, which can be
parameterised by $(\v^m, b^m) \in \B \times \R$. Let $M = \margin\, H(\v^m, b^m)$,
$C^+ = \{\x \in \X \,\vert \, \v^m \cdot \x - b^m = M\}$ and $C^- = \{\x \in \X \,
\vert \, b^m - \v^m \cdot \x = M\}$.
Then, $\forall \w \in \mathrm{Null}(\v^m)$, $c \in \R$ either $\min\{\w\cdot \x - c
\, \vert \, \x \in C^+\}\leqslant 0$, or $\max\{\w\cdot \x - c \, \vert \, \x \in
C^-\}\geqslant 0$.

\end{lemma}

\begin{proof}

Suppose the result does not hold, then $\exists (\w, c)$ with $\|\w\| = 1, \w
\cdot \v^m = 0$ and $\min\{\w \cdot \x - c \vert \x \in C^+\} > 0$ and $\max\{\w \cdot
\x - c \vert \x \in C^-\} < 0$.
Let $m = \min\{ |\w\cdot \x - c | \,\big \vert\, \x \in C^+ \cup C^-\}$.  Define $\lambda = \frac{m}{2M} < 1$.
Define $\u = \frac{1}{\sqrt{\lambda^2 + (1-\lambda)^2}}(\lambda \w + (1-\lambda)
\v^m)$ and $d = \frac{\lambda c + (1-\lambda) b^m}{\sqrt{\lambda^2 +
(1-\lambda)^2}}$. By construction $\|\u\| = 1$. For any $\x_+ \in C^+$ we have,
\begin{align*}
\u \cdot \x_+ - d &= \frac{\lambda(\w \cdot \x_+ - c) + (1-\lambda)(\v^m \cdot \x_+ -
b^m)}{\sqrt{\lambda^2 + (1-\lambda)^2}}\\
&\geqslant \frac{\lambda m + (1-\lambda)M}{\sqrt{\lambda^2 + (1-\lambda)^2}}\\
&= \frac{m^2 + 2M^2 - Mm}{\sqrt{m^2 + (2M-m)^2}}\\
&>M.
\end{align*}
Similarly one can show that $d - \u \cdot \x_- > M$ for any $\x_- \in C^-$,
meaning that $(\u,d)$ achieves a larger margin on $C^+$ and $C^-$ than $(\v^m, b^m)$, a
contradiction.

\end{proof}

\noindent
The next lemma uses the above result to provide an upper bound on the
distance between pairs of support points projected onto any vector, in terms of
the angle between that vector and $\v^m$.

\begin{lemma} \label{lm:2Mv}

Suppose there is a unique hyperplane in $F$ with maximum margin, which can
be parameterised by $(\v^m, b^m) \in \B \times \R$. Define $M =
\margin\, H(\v^m,b^m)$, $C^+ = \{\x \in \X \vert \v^m \cdot \x - b^m = M\}$, and $C^- = \{\x \in
\X \vert  b^m-\v^m\cdot \x = M\}$.
There is no vector $\w \in \R^d$ for which $\w\cdot \x_+ - \w\cdot \x_- > 2M \v^m \cdot \w$ for
all pairs $\x_+ \in C^+, \x_- \in C^-$. 

\end{lemma}

\begin{proof}

Suppose such a vector exists. Define $\w' = \w - (\v^m\cdot \w)\v^m$. By construction $\w' \in
\mathrm{Null}(\v^m)$. For any pair $\x_+ \in C^+, \x_- \in C^-$ we have
\begin{align*}
\w' \cdot x_+ - \w' \cdot x_- &= \w \cdot x_+ - (\v^m \cdot \w) \v^m \cdot x_+ - \w
	\cdot x_- +  (\v^m \cdot \w) \v^m \cdot x_-\\
&> \v^m\cdot \w(2M - \v^m\cdot x_+ + b^m - b^m + \v^m \cdot x_-)\\
&= 0.  \end{align*}
Define $c: = \frac{1}{2}(\min \{ \w' \cdot \x_ + \big \vert \x_+ \in C^+\} + \max
\{ \w' \cdot \x_- \big \vert \x_- \in C^-\})$. Then $\min\{\w'\cdot \x_+ - c \vert
\x_+ \in C^+\} > 0$ and $\max\{\w' \cdot \x_- - c \vert \x_- \in C^-\} < 0$, a
contradiction.

\end{proof}

We are now in a position to provide the main proof of this appendix. The theorem states that if the maximum
margin hyperplane is unique, and can be parameterised by $(\v^m, b^m) \in \B\times \R$, then
\[
\lim_{h \to 0^+} \min \left\{\|(\v^\star_h, b^\star_h) - (\v^m, b^m)\|, \|(\v^\star_h, b^\star_h) + (\v^m, b^m)\|\right\} = 0,
\]
where $\left\{H(v^\star_h, b^\star_h)\right\}_h$ is any collection of minimum density hyperplanes indexed by their bandwidth $h > 0$.

\vspace{0.5cm}
\begin{proof}{{\bf of Theorem~\ref{thm:convergence}}}

Define $M = \margin\, H(\v^m, b^m)$, $C^+ = \{\x \in \X \,\vert \, \v^m \cdot \x - b^m = M\}$
and $C^- = \{\x \in \X \, \vert \, b^m - \v^m \cdot \x = M\}$. Let $B = \max\{\| \x\| \, \big \vert \, \x \in \X\}$. Take any $\epsilon > 0$ and set $0<\delta$ to satisfy $\frac{2 \delta}{M}(1+B^2) + 2B\delta^{3/2}\sqrt{\frac{2}{M}} + \delta^2 = \epsilon^2$. Now, suppose $(\w, c) \in \B \times \R$ satisfies,
\begin{equation*}
\w \cdot \x_+ - c > M- \delta, \ \forall \x_+ \in C^+ \mbox{ and } c - \w \cdot \x_- > M - \delta, \ \forall \x_- \in C^-. 
\end{equation*}
By Lemma \ref{lm:2Mv} we know that $\exists \x_+ \in C^+, \x_- \in C^-$ s.t. $\w \cdot \x_+ -  \w\cdot \x_- \leqslant 2M \v^m \cdot \w$. Thus
\begin{align*}
\v^m \cdot \w &\geqslant \frac{\w\cdot \x_+ - \w \cdot \x_-}{2M}\\
&= \frac{\w\cdot \x_+ -c + c- \w \cdot \x_-}{2M}\\
&> \frac{2M - 2 \delta}{2M} = 1 - \frac{\delta}{M}.
\end{align*}
Thus $\|\v^m - \w \|^2 < \frac{2\delta}{M}$. Now, for each $\x_+ \in C^+, \v^m \cdot \x_+ - b = M$ and for each $\x_- \in C^-, b - \v^m \cdot \x_- = M$. Thus for any such $\x_+, \x_-$ we have,
\begin{eqnarray*}
&M - \delta + \w \cdot \x_- < c < \w \cdot \x_+ - M + \delta,\\
& b^m - \v^m \cdot \x_- - \delta + \w \cdot \x_- < c < \w \cdot \x_+ - \v^m \cdot \x_+ + b^m + \delta,\\
& b^m-\delta - (\v^m - \w)\cdot \x_- < c < b^m + \delta + (\w - \v^m)\cdot \x_+,\\
& b^m-\delta - B\|\v^m - \w\| < c < b^m+\delta + B\|\w-\v^m\|, \\
& | c - b^m | < \left| \delta + B\|\w-\v^m\| \right|.
\end{eqnarray*}
We can now bound the distance between $(\w,c)$ and $(\v^m,b^m)$,
\begin{eqnarray*}
\|(\v^m, b^m) - (\w, c)\|^2 &=& \|\v^m-\w\|^2 + \vert b^m -c\vert ^2\\
&<& \|\v^m-\w\|^2(1+B^2) + 2B\delta\|\v^m-\w\| + \delta^2\\
&<& \frac{2 \delta}{M}(1+B^2) + 2B\delta\sqrt{\frac{2 \delta}{M}} + \delta^2\\
&=& \epsilon^2.
\end{eqnarray*}

\noindent
We have shown that for any hyperplane $H(\w, c)$ that achieves a margin larger than $M-\delta$
on the support points of the maximum margin hyperplane, $\x \in C^+ \cup C^-$, the distance
between $(\w,c)$ and $(\v^m,b^m)$ is less than $\epsilon$.
Equivalently, any hyperplane $H(\w,
c)$ such that $\| (\w,c) - (\v^m, b^m)\| > \epsilon$ has a margin less than $M-\delta$, as
$\min\left\{ | \w \cdot \x - c| \, \big\vert \, \x \in C^+ \cup C_- \right\} < M-\delta$.
By symmetry, the same holds for any $(\w,c)$ within distance $\epsilon$ of $(-\v^m, -b^m)$.

By Lemma \ref{lm:samedivision} $\exists h_1>0$
such that for all $h \in (0,h_1)$, the minimum density hyperplane for $h$, $H(\v^\star_h, b^\star_h)$,
induces the same partition
of $\X$ as the maximum margin hyperplane, $H(\v^m, b^m)$.
By Lemma \ref{lm:bddmargin} $\exists h_2 > 0$ such that for all $h \in(0,h_2)$,
$\margin\, H(\v^{\star}_h, b^{\star}_h) > M-\delta$.
Therefore for $h \in (0,\min\{h_1, h_2\})$, $\min\left\{ \|(\v^{\star}_h, b^{\star}_h) - (\v^m, b^m)\|,
\| (\v^{\star}_h, b^{\star}_h) +(\v^m, b^m)\| \right\} < \epsilon$. Since $\epsilon > 0$
was arbitrarily chosen, this gives the result. 

\end{proof}

\end{document}